\documentclass[11pt,a4paper]{article}

\usepackage{amsmath,amssymb,amsthm}
\usepackage[margin=2.5cm]{geometry}
\usepackage{enumitem}
\usepackage{array}
\usepackage{cite}
\usepackage[hidelinks]{hyperref}
\usepackage{bbm}
\usepackage{booktabs}
\usepackage[table]{xcolor}
\usepackage{needspace}
\usepackage{bm}
\usepackage{tikz}
\usepackage{tabularx}

\usetikzlibrary{positioning, fit, backgrounds, arrows.meta, calc}

\definecolor{truepos}{RGB}{240,247,240}
\definecolor{miss}{RGB}{252,232,232}
\definecolor{falsepos}{RGB}{235,241,251}
\definecolor{trueneg}{RGB}{247,247,247}

\theoremstyle{plain}
\newtheorem{theorem}{Theorem}
\newtheorem{assumption}[theorem]{Assumption}
\newtheorem{proposition}[theorem]{Proposition}
\newtheorem{corollary}[theorem]{Corollary}
\newtheorem{lemma}[theorem]{Lemma}
\theoremstyle{definition}
\newtheorem{definition}[theorem]{Definition}

\newtheorem{remark}[theorem]{Remark}

\title{Finite-Sample Coverage Audits for High-Recall Candidate Generation:\\[4pt]
Certification and Learning-Theoretic Design}

\author{
Martin Anthony$^{1,2}$ \qquad
Kaveh Salehzadeh Nobari$^{1,3}$\\[6pt]
\small $^{1}$Data Science Institute,
London School of Economics and Political Science\\
\small $^{2}$Department of Mathematics,
London School of Economics and Political Science\\
\small $^{3}$The Inclusion Initiative,
London School of Economics and Political Science
}

\date{}

\begin{document}
\maketitle

\begin{abstract}
An initial high-recall stage in an empirical pipeline decides which items pass to later review, labelling, or modelling, and relevant items it misses are lost to every subsequent stage. We study how many audit labels are needed to certify, with finite-sample validity, that this missed relevant mass is small, and our main results
characterise the label complexity of this problem. We first show that no procedure using only labels from inside the candidate set can certify any non-trivial bound on the missed mass: the audit must sample the excluded pool, the only region where unrecovered relevant items can lie. We then prove a matching finite-corpus lower
bound. Any valid audit that certifies fewer than $m$ missed relevant items with high probability when none are present, even if adaptive and permitted to label the entire included pool, must inspect on the order
of $N_0/m$ excluded-pool labels. Excluded-pool auditing is therefore minimax rate-optimal, not merely
convenient, for missed-mass certification in the zero-miss regime. Building on this characterisation, we develop an exact finite-sample toolkit, using binomial and hypergeometric inversion rather than asymptotic approximation, that certifies missed mass, converts it to recall through a two-pool design, certifies pre-specified families of nested candidate generators simultaneously, and produces stress-test certificates against declared perturbation mechanisms. These certificates can be paired with observable review burden to
select the least burdensome pre-specified candidate generator meeting a missed-mass target. Every guarantee holds under one discipline: the candidate generator, or the pre-specified family from which it is
selected, and the audit rule are fixed before the certification labels are examined.
\end{abstract}


\section{Introduction}
\label{sec:intro}

Many machine learning and data-engineering pipelines begin with a high-recall candidate-generation step. Given a large corpus and a latent relevance criterion, an initial candidate generator selects a subset of items for later review, labelling, or downstream modelling. The intended role of this first stage is coverage: it should capture nearly all relevant items, even at the cost of admitting many irrelevant ones. Excluding relevant observations from the training sample can induce sample-selection bias when the exclusion mechanism is non-random, because the empirical training distribution may no longer represent the target population. If the exclusions are random, the main effect is typically loss of precision rather than bias \cite{LittleRubin2019,Heckman1979,LiuZiebart2014}.

This class of problems has a structurally asymmetric error profile. Relevant items missed at the candidate-generation stage are irrecoverable: they cannot be reviewed, cannot be labelled, and cannot be learned by any downstream model trained on the resulting candidate set. Over-inclusion, by contrast, increases review burden but does not destroy information. Examples include responsiveness review in legal discovery, study screening in systematic reviews, compliance and regulatory audit, and pipelines in which annotated training data is derived from a filtered corpus.

In technology-assisted legal review and systematic-review screening, for instance, an initial search or classifier returns documents or studies for human assessment. Items not returned may never be examined, so the quantity of interest is the relevant material left outside the returned set. The objective in these settings is to construct a candidate set $\widehat S\subseteq\mathcal X$ that captures nearly all latent relevant instances with high probability while keeping the review cost tolerable, where $\mathcal X$ denotes the population, or item, space. The candidate set is then handed to human annotators, downstream classifiers, or both. The statistical task is to certify, from audit data, that the missed relevant mass is small.

Many audit summaries report estimated recall together with approximate confidence intervals. Our aim is different: we derive finite-sample one-sided certificates for missed relevant mass and recall, using binomial or hypergeometric inversion according to the sampling design.

Several existing frameworks provide useful ingredients. Classical PAC learning provides distribution-free guarantees on generalisation error \cite{Valiant1984,AnthonyBartlett1999}. Candidate generation asks instead for one-sided coverage of a latent relevant set. Recent distribution-free risk-control frameworks (e.g., Learn then Test \cite{LTT2021}, Conformal Risk Control \cite{CRC2024}) provide general finite-sample machinery for controlling expected losses, including risks such as false negative rate. We specialise these ideas to high-recall candidate generation, where the central statistical design choice is where to spend labels.

We first analyse fixed candidate generators, then extend the guarantees to finite families of generators chosen after seeing non-audit data. The certification audit remains untouched while the candidate generator is chosen. We may use separate design labels to choose a candidate generator from a structured class, but the final certificate uses only the held-out audit labels. In Section~\ref{sec:learning} we record the finite-class, VC, sparse-union, and Neyman--Pearson guarantees needed for this design stage, and we show how to pair them with independent audit certification. In Section~\ref{sec:worked_example} we give a compact numerical example using only deterministic audit calculations.

\subsection{Motivating examples}

We give two motivating examples drawn from opposite ends of the modelling
process. The first is \emph{pre-transformer}: a data-engineering pipeline that
prepares a corpus for human annotation, of the kind later used to train or
evaluate models. The second is \emph{post-transformer}, arising at inference
time: the retrieval stage of retrieval-augmented generation (RAG), in which
relevant passages are supplied to a language model to ground its output. The
two settings look unrelated, but both instantiate the same object, namely a
filter over a corpus whose items carry a latent relevance label. In both, a
relevant item missed by the filter is irrecoverable at every later stage, and
this asymmetry is what the certificates in this paper are designed to control.

\paragraph{A pre-transformer instance: building annotation data.}
Consider a pipeline for building labelled data on sentences related to a
specified concept, such as diversity, equity, and inclusion disclosures in
annual reports of publicly traded firms. The documents are long, and the
relevant concept may be expressed in heterogeneous, indirect, and
context-dependent language. Before any large-scale review, measurement, or
downstream modelling can be carried out, the analyst must decide which sentences
will be sent for human annotation.

A common first step is to apply a candidate generator $g$, which selects
sentences deemed potentially relevant. The generator may use keyword lists,
regular expressions, similarity thresholds, embedding searches, supervised
classifiers, or combinations of such rules. Human annotators then label the
selected sentences as relevant or not. The resulting labels may be used as the
final annotated dataset, or they may later be used to train a supervised
classifier.

The audit problem concerns the candidate generator $g$. If $g$ excludes relevant
sentences, those sentences are unavailable to later review, and they are absent
from any labelled data built from the selected set. The purpose of the audit is
therefore to certify that the relevant mass left outside the candidate set is
small.

When the candidate generator is itself learned, the same distinction applies. We
may use design labels to train or tune a classifier $\widehat{\varphi}$, and then
use this classifier as the candidate generator. The held-out audit labels are
not used to choose $\widehat{\varphi}$. They are used only to certify the missed
relevant mass, or recall, of the resulting selection rule.

The labelled candidates do not by themselves contain the information needed to
certify missed mass: annotators see only the selected items and observe no labels
for excluded ones. Proposition~\ref{obs:no_inside_only} makes the resulting
impossibility precise, and motivates the finite-sample audit formulation.

\paragraph{A post-transformer instance: retrieval-augmented generation.}
The same asymmetry appears in the retrieval stage of retrieval-augmented
generation (RAG), in which a retriever selects a small set of passages from a
large corpus and a language model conditions its answer on that selected
set~\cite{LewisEtAl2020}. Here the items being evaluated are query--passage
pairs $(q,x)$, where $q$ is a query and $x$ is a passage from the corpus. The
candidate generator $g(q,x)$ is the retriever together with its cutoff, with
$g(q,x)=1$ when passage $x$ is retrieved for query $q$, and the latent
relevance label $\varphi^\star(q,x)$ indicates whether $x$ contains evidence
required to answer $q$ correctly. If the retriever fails to return a relevant
passage for a given query, that passage is absent from the model's context and
cannot be used as retrieved evidence to ground the subsequent answer. This is
the irrecoverable false negative of the present setting, and it is distinct
from the much-studied failure in which relevant evidence \emph{is} retrieved
but the generator ignores or contradicts it~\cite{Elchafei2026}. That latter
failure is a question of generation faithfulness rather than coverage, and our
certificates do not address it. What an excluded-pool audit certifies is
precisely the coverage side: sampling query--passage pairs for which
$g(q,x)=0$ and labelling whether $\varphi^\star(q,x)=1$ bounds the relevant
evidence mass left outside the retrieved sets, with the finite-sample
guarantees of Section~\ref{sec:excluded}. As in the pre-transformer example,
relevance is a fixed property of a query--passage pair and is checkable by an
oracle, while an excluded relevant passage is unavailable to the downstream
generator as retrieved evidence. These two conditions are what make the
missed-mass certificate meaningful. Because relevant query--passage pairs may be rare in a large retrieval
corpus, conditional targets such as recall, positive-conditional audits, or per-query stratified certificates may be more operationally meaningful in this setting than absolute missed mass alone.

\subsection{Contributions}
Our main results characterise the label complexity of missed-mass certification.
\begin{enumerate}[leftmargin=2em]
    \item \textbf{Impossibility:} 
Labels from inside the candidate set, together with the generator values and unlimited unlabelled data, cannot certify any non-trivial bound on missed mass (Proposition~\ref{obs:no_inside_only}). The audit must sample the excluded pool.

\item \textbf{Matching lower bound and optimality:}
In the finite-corpus model, any valid audit that successfully rules out $m$ missed relevant items with high probability in the zero-miss case must inspect $\Omega(N_0/m)$ distinct excluded-pool items on average, even if the audit is adaptive, randomised, and permitted to label the entire included pool (Theorem~\ref{thm:label_lower_bound}). A simple uniform excluded-pool audit that observes zero relevant items achieves the same $N_0/m$ scaling, up to a constant factor for fixed error parameters. Thus, excluded-pool auditing is minimax rate-optimal for missed-mass certification in the zero-miss regime (Corollary~\ref{cor:minimax}).
\end{enumerate}
The remaining contributions turn this principle into an exact, deployable audit workflow.
\begin{enumerate}[leftmargin=2em,resume]
    \item Exact finite-sample certification for excluded-pool, finite-corpus, two-pool recall, and pre-specified prefix audits (Section~\ref{sec:excluded}).
    \item Operational tools: sample-size calculations, burden-minimising stopping, shared-reference designs, and fixed-sequence certification without a multiplicity penalty (Sections~\ref{sec:samplesize},~\ref{sec:fixed_sequence}).
    \item Auditable diagnostic and stress-test certificates (Sections~\ref{sec:amplification},~\ref{sec:stress}).
    \item A learning-theoretic design layer with independent held-out certification (Section~\ref{sec:learning} and Appendix~\ref{app:learning}).
    \item A numerical worked example (Section~\ref{sec:worked_example}).
\end{enumerate}

\subsection{Relation to existing work}

Learn then Test~\cite{LTT2021} and Conformal Risk Control~\cite{CRC2024}
provide general distribution-free risk-control templates. The contribution here is
more specialised: the audit uses the geometry of candidate generation. Crucially, those general templates provide risk control for fixed procedures or pre-specified families, but do not characterise how many
labels this candidate-generation problem requires or where those labels must be drawn. The impossibility result and the lower bound are questions about the candidate-generation structure itself, and have no analogue in the general risk-control machinery. Missed items can occur only in the excluded pool, the excluded-pool mass \(p_0(g)\) is observable without relevance labels, and pre-specified prefix families give
natural targets for simultaneous certification. The excluded-pool certificates
exploit the decomposition \(r(g)=p_0(g)\eta(g)\), where $\eta(g)$ is the fraction of those excluded items that are truly relevant, rather than treating misses as
rare events in the whole population.

The closest applied lineage is technology-assisted review and systematic-review
screening. Work in e-discovery and high-recall review developed relevance
feedback, continuous active learning, stopping rules, and recall validation
methods~\cite{CormackMojdeh2009,CormackGrossman2014,CormackGrossman2015,
CormackGrossman2016,LeaseEtAl2016,MagdyJones2010,Webber2013,
LewisYangFrieder2021,YangLewisFrieder2021,SneydStevenson2021}. The certificates
here formalise the same sampling geometry with exact finite-sample one-sided
statements, including excluded-pool missed-count certification, two-pool recall
certification, and simultaneous certification of pre-specified prefix families.
A related validation design in technology-assisted review estimates found and
missed relevant items from separate samples~\cite{GrossmanCormackVetting2021}.
Our results give distribution-free finite-sample versions of this idea.
Webber~\cite{Webber2013} shows that the normal approximation to recall covers
poorly in this setting, studies two-tailed approximate intervals that target
mean rather than guaranteed coverage, and lists one-sided lower bounds and
guaranteed-coverage methods as future work. The two-pool certificates of
Section~\ref{sec:excluded} supply guaranteed-coverage one-sided bounds of
exactly this kind. Closest to the
finite-corpus certificates below, Callaghan and
M\"uller-Hansen~\cite{CallaghanMullerHansen2020} use a hypergeometric test as a
stopping criterion that certifies a recall target at a stated confidence in
active-learning screening. Relative to that line of work, the present paper
adds i) the included-pool impossibility result (Proposition~\ref{obs:no_inside_only}), ii) the matching label-complexity lower bound and the resulting minimax optimality of excluded-pool auditing (Theorem~\ref{thm:label_lower_bound}, Corollary~\ref{cor:minimax}), iii) simultaneous certification of pre-specified prefix families, iv) the exact two-pool combination for recall, and v) the explicit separation of design from certification (Assumption~\ref{ass:(A1)}). Items (i) and (ii) are specific to the candidate-generation geometry and have no counterpart in a hypergeometric stopping rule. The certificates here do not rely on a model of the rate at which relevant items appear.

The design-stage material uses standard learning-theoretic tools. The
bicriteria formulation is the Neyman--Pearson approach~\cite{ScottNowak2005,
CannonEtAl2002,Tong2013,RigolletTong2011} transposed to candidate
generation:\footnote{The empirical bicriteria formulation was first proposed
by Cannon, Howse, Hush, and Scovel in a 2002 Los Alamos technical report
(LA-UR-02-2951). The original Los Alamos host is no longer available, though
the report remains retrievable through the Internet Archive. The rigorous
published treatment is given by \cite{ScottNowak2005}.}
review burden is minimised subject to a one-sided missed-mass constraint. PAC and
VC bounds~\cite{Valiant1984,Natarajan1987,AnthonyBartlett1999,
BlumerEhrenfeuchtHausslerWarmuth1989} control the selection of a candidate
generator from a structured class before independent audit certification.

There is also a conceptual connection with set estimation~\cite{DevroyeWise1980,
Polonik1995,Tsybakov1997,ScottNowak2006,WillettNowak2007}. The candidate set is
asked to cover most of the latent relevant-item distribution while keeping burden
small. The difference is operational: here the set being certified is fixed, or
pre-specified, before the certification labels are opened, and the audit can be
targeted to the excluded pool.

\section{Problem setup}
\label{sec:setup}

We use two formulations. In the i.i.d.\ population model, items are drawn from an underlying distribution $P$. In the finite-corpus model, we treat the $N$ corpus items, the relevance labels, and the candidate generator as fixed. The probability statements refer only to the random choice of audit sample. We use the population formulation for the main theoretical statements. We use the finite-corpus formulation for sampling-without-replacement certificates and for the connection to existing technology-assisted review practice.

\subsection{Population model}

Let $(\mathcal X,\mathcal F,P)$ be the population probability space.
We use the canonical random element
\[
X:(\mathcal X,\mathcal F,P)\longrightarrow(\mathcal X,\mathcal F),
\qquad
X(\omega)=\omega.
\]
Thus the outcome $\omega$ is itself the sampled population item, and
the law of $X$ is $P$. Let
\[
\varphi^\star:
(\mathcal X,\mathcal F)
\longrightarrow
\left(\{0,1\},2^{\{0,1\}}\right)
\]
be a measurable latent relevance indicator, and define
\[
Y:=\varphi^\star(X).
\]
Equivalently, \(Y(x)=\varphi^\star(x)\) for \(x\in\mathcal X\).
The induced joint law of \((X,Y)\) on
\[
\left(
\mathcal X\times\{0,1\},
\mathcal F\otimes 2^{\{0,1\}}
\right)
\]
is the pushforward measure
\[
P_{X,Y}
:=
P\circ T^{-1},
\qquad
T(x):=\left(x,\varphi^\star(x)\right).
\]

The function \(\varphi^\star\) is fixed but unknown to the analyst.
It represents the ground-truth criterion determining whether an item is
relevant, for example legal responsiveness, study eligibility, or policy
applicability. The analyst observes \(\varphi^\star\) only through audit
labels on sampled items.

Write
\[
S^\star:=\{x\in\mathcal X:\varphi^\star(x)=1\},
\qquad
\pi:=P(S^\star)=P(Y=1).
\]
A candidate generator is a measurable function
\[
g:
(\mathcal X,\mathcal F)
\longrightarrow
\left(\{0,1\},2^{\{0,1\}}\right),
\]
with candidate set
\[
\widehat S:=\{x\in\mathcal X:g(x)=1\}.
\]

For population events, \(P\) always denotes the probability measure on
the canonical space \((\mathcal X,\mathcal F)\). Thus, for example,
\[
P(Y=1,g(X)=0)
:=
P\left(
\{x\in\mathcal X:
\varphi^\star(x)=1,\ g(x)=0\}
\right).
\]
Equivalently, such probabilities may be evaluated under the induced
joint law \(P_{X,Y}\), but we suppress that subscript throughout.

Audit, design-sample, and algorithmic randomness are carried on
probability spaces separate from the population space. For a given audit
design, let
\[
(\Omega_{\mathrm A},\mathcal A_{\mathrm A},
 \mathbb{P}_{\mathrm A})
\]
denote the probability space carrying the random audit sample and any
independent randomisation used by the audit procedure. The population
objects \(P\), \(\varphi^\star\), and the candidate generator or generator
family being certified are fixed under \(\mathbb{P}_{\mathrm A}\).

We write \(\mathbb{P}_{\mathrm{audit}}\) for the probability measure induced by
the audit design and \(\mathbb{P}_{\mathrm{design}}\) for the probability measure
induced by a separate design sample. These symbols are theorem-specific:
the sampling scheme stated in each result determines the corresponding
probability measure.

Throughout, $P$ denotes the population law of $X$, and hence the
induced law of $(X,Y)$. We use $\mathbb{P}$ for probabilities over auxiliary
randomness, such as audit samples, design samples, randomised procedures,
and binomial or hypergeometric variables. Subscripts indicate the source
of randomness when it is important, for example
$\mathbb{P}_{\mathrm{audit}}$ and $\mathbb{P}_{\mathrm{design}}$.

For a candidate generator $g$, define
\begin{align*}
r(g) &:= P(Y=1,\, g(X)=0), \\
L(g) &:= P(g(X)=0 \mid Y=1), \\
\mathrm{Recall}(g) &:= P(g(X)=1 \mid Y=1) = 1 - L(g), \\
b(g) &:= P(Y=0,\, g(X)=1), \\
p_0(g) &:= P(g(X)=0),\\
p_1(g)\equiv B(g) &:= P(g(X)=1)=1-p_0(g).
\end{align*}
The conditional quantities $L(g)$ and $\mathrm{Recall}(g)$ are defined when $\pi=P(Y=1)>0$. The quantity $r(g)$ is the \emph{absolute missed relevant mass}, $L(g)$ is the \emph{conditional miss rate} (so $\mathrm{Recall}(g) = 1 - L(g)$ is the \emph{recall} of $g$), $b(g)$ is the \emph{false-positive mass} (wasted review), $B(g)$ is the \emph{review burden} (the probability that an item is sent for review), and $p_0(g) = 1 - B(g)$ is the \emph{excluded-pool mass}. When $\pi > 0$, $r(g) = \pi L(g)$, so an upper bound $u$ on $r(g)$ implies $L(g) \le u/\pi$ when $\pi$ is known.

Table~\ref{tab:population-confusion-matrix} displays these quantities
as probability masses in a candidate-generation confusion matrix.

\begin{table}[htbp]
\centering
\small
\renewcommand{\arraystretch}{1.55}
\setlength{\tabcolsep}{7pt}

\medskip

\begin{tabularx}{\textwidth}{
  @{}
  >{\raggedright\arraybackslash}p{2.25cm}
  >{\centering\arraybackslash}X
  >{\centering\arraybackslash}X
  >{\centering\arraybackslash}p{1.45cm}
  @{}
}
\toprule
&
\multicolumn{2}{c}{\textbf{Candidate generator}}
&
\\[-1mm]

\cmidrule(lr){2-3}

\textbf{Ground truth}
&
\shortstack{
  \textbf{$g(X)=1$}\\
  {\footnotesize included for review}
}
&
\shortstack{
  \textbf{$g(X)=0$}\\
  {\footnotesize excluded from review}
}
&
\textbf{Total}
\\

\midrule

\shortstack[l]{
  \textbf{$Y=1$}\\
  {\footnotesize relevant}
}
&
\shortstack{
  \textbf{True positive}\\
  {\footnotesize recovered relevant mass}\\[2pt]
  $P(Y=1,g(X)=1)$\\
  $=\pi-r(g)=\pi[1-L(g)]$
}
&
\cellcolor{red!7}
\shortstack{
  \textbf{False negative}\\
  {\footnotesize missed relevant mass}\\[2pt]
  $P(Y=1,g(X)=0)$\\
  $\boldsymbol{=r(g)=\pi L(g)}$
}
&
$\pi$
\\

\addlinespace[4pt]

\shortstack[l]{
  \textbf{$Y=0$}\\
  {\footnotesize not relevant}
}
&
\shortstack{
  \textbf{False positive}\\
  {\footnotesize unnecessary review mass}\\[2pt]
  $ P(Y=0,g(X)=1)$\\
  $=b(g)$
}
&
\shortstack{
  \textbf{True negative}\\
  {\footnotesize correctly excluded mass}\\[2pt]
  $ P(Y=0,g(X)=0)$\\
  $=p_0(g)-r(g)$
}
&
$1-\pi$
\\

\midrule

\textbf{Total}
&
\shortstack{
  $B(g)$\\
  {\footnotesize review burden}
}
&
\shortstack{
  $p_0(g)$\\
  {\footnotesize excluded-pool mass}
}
&
$1$
\\

\bottomrule
\end{tabularx}

\medskip

\[
L(g)
=
\frac{\text{false-negative mass}}
     {\text{total relevant mass}}
=
\frac{r(g)}{\pi},
\qquad
\operatorname{Recall}(g)=1-L(g).
\]

\caption{Probability-mass confusion matrix for a candidate generator.}
\label{tab:population-confusion-matrix}
\end{table}

The distinction matters in rare-relevance settings. If $\pi = 0.01$ and an audit certifies $r(g) \le 0.005$, the implied bound is $L(g) \le 0.5$, equivalently a recall of at least $0.5$, which is not a high-recall guarantee.

A candidate generator may be a single rule, a single classifier, or a union of $J$ component detectors $h_1, \ldots, h_J : \mathcal X \to \{0,1\}$. We use the term component detector for the primitive functions $h_j$ from which some candidate generators are built. The main object certified by the audit is always the candidate generator $g$. For any $t\in\{1,\ldots,J\}$, define the prefix union
$$
g_t(x):=\max_{j\le t} h_j(x).
$$
In practice, the components are often constructed sequentially, with each new $h_j$ designed to recover relevant items missed by the partial union of those built so far. In other words, the candidate set expands monotonically, i.e.,
\[
\widehat S_1 \subseteq \widehat S_2\subseteq \ldots\subseteq \widehat S_J .
\]
We call $g_t$ a \emph{prefix} of the ordered sequence, and $g_J$ is the full union. Adding components to a union can only reduce missed relevant mass (and conditional miss rate), at the cost of an increase in review burden. The certification theorems below treat $g$ as an arbitrary fixed measurable function. We use the prefix structure later to define stopping rules along the ordered sequence and to certify all prefixes with a single shared audit sample. We reserve $J$ for the number of component detectors and use $M$ for the number of pre-specified prefixes considered for certification, since the number of prefixes considered need not equal $J$. The subscripted counts $M_0(g)$ and $M_1(g)$ and the inversion bound $M_U$ defined below are distinct from $M$. The symbol $Q$ always denotes the size of the audit sample under discussion.

\subsection{Finite-corpus model}

In applications, the population is typically a finite corpus $\mathcal X_N = \{x_1, \ldots, x_N\}$. For a fixed candidate generator $g:\mathcal X_N\to\{0,1\}$, write
\[
z_i=g(x_i),\qquad y_i=\varphi^\star(x_i),
\qquad i=1,\ldots,N.
\]
Thus $z_i$ records whether item $x_i$ is selected by the generator, while $y_i$ records whether it is relevant. Write
\begin{align*}
N_0(g) &= |\{i : z_i = 0\}|, &
N_1(g) &= |\{i : z_i = 1\}|, \\
M_0(g) &= |\{i : z_i = 0,\, y_i = 1\}|, &
M_1(g) &= |\{i : z_i = 1,\, y_i = 1\}|,
\end{align*}
where $|A|$ denotes the cardinality of set $A$. The candidate-generator values $z_i$ are observable for every item in $\mathcal X_N$. The relevance labels $y_i$ are observable only on items selected for audit. The finite-corpus empirical fractions, the analogues of the population probabilities, are
$$
p_{0,N}(g)=\frac{N_0(g)}{N}, \qquad B_N(g)=\frac{N_1(g)}{N}.
$$
Thus $p_{0,N}(g)$ corresponds to $p_0(g)=P(g(X)=0)$, and $B_N(g)$ corresponds to $B(g)=P(g(X)=1)$, under the i.i.d.\ population model. The finite-corpus recall, defined when $M_0(g)+M_1(g)>0$, is
$$
\mathrm{Recall}(g)=\frac{M_1(g)}{M_0(g)+M_1(g)}.
$$

\subsection{Audit model}
\label{sec:audit_model}

Operationally, an \emph{audit} is the mechanism that converts selected unlabelled items into labelled items. In the finite-corpus model, before auditing, the corpus $\mathcal X_N=\{x_1,\ldots,x_N\}$ is fixed. We use the notation $z_i=g(x_i)$ and $y_i=\varphi^\star(x_i)$ introduced above. The values $z_i$ are known for every item, while the relevance labels $y_i$ are latent. An audit chooses a subset
$$
S \subseteq \{1,\ldots,N\},
$$
possibly at random according to a declared sampling design, sends those items to a labelling oracle, and returns
$$
\{(x_i,y_i): i\in S\}.
$$
In the i.i.d.\ population model, the same object is a labelled sample $(X_i,\varphi^\star(X_i))$ drawn from a specified sampling distribution.

Each result below uses a specified sampling design and certifies a specified error quantity. The table gives the informal content of the main guarantees.
\begin{center}
{\small
\setlength{\tabcolsep}{4pt}
\renewcommand{\arraystretch}{1.08}
\begin{tabular}{%
>{\raggedright\arraybackslash}p{0.14\linewidth}%
>{\raggedright\arraybackslash}p{0.44\linewidth}%
>{\raggedright\arraybackslash}p{0.36\linewidth}}
\hline
Result & Sampling design & Certified quantity \\
\hline
Theorem~\ref{thm:excluded_binomial}
& Independent draws from the excluded pool, that is, from the items not selected by the candidate generator.
& An upper bound on the relevant mass missed by the generator. \\

Theorem~\ref{thm:excluded_hypergeometric}
& Uniform sampling without replacement from the excluded pool in a fixed finite corpus.
& An upper bound on the number of relevant items missed by the generator. \\

Theorem~\ref{thm:two_pool}
& Independent samples from the excluded pool and the included pool.
& A lower bound on recall, obtained by combining the two audit samples. \\

Theorem~\ref{thm:whole_population}
& Independent draws from the whole population, before conditioning on whether the generator selects the item.
& An upper bound on the relevant mass missed by the generator. \\

Corollary~\ref{cor:prefix_stopping}
& Separate excluded-pool audits for a pre-specified finite list of prefix generators.
& Simultaneous upper bounds on the missed relevant mass of all prefixes in the list. \\

Theorem~\ref{thm:shared_reference_prefix}
& A single labelled reference sample, used to evaluate a fixed sequence of candidate generators.
& Simultaneous missed-mass guarantees for all generators in the sequence. \\

Theorem~\ref{thm:fixed_sequence}
& Sequential audits along a pre-specified nested sequence of candidate generators.
& A valid stopping rule that returns a generator meeting a pre-specified missed-mass target, when such a generator is found. \\

Theorem~\ref{thm:auditable_amplification}
& An audit of relevant items, used to estimate how often each component detector captures a relevant item.
& An upper bound on the conditional miss rate, equivalently a lower bound on recall, after combining several component detectors.  \\

Theorem~\ref{thm:stress_test}
& An audit of relevant items together with independently generated variants of those items.
& Bounds on how often relevant items, or their variants, escape detection. \\
\hline
\end{tabular}
}
\end{center}
We obtain every entry in the table by exact binomial or hypergeometric inversion.

All these designs share two assumptions.

\begin{assumption}[Design-Certification Separation]\label{ass:(A1)}
The candidate generator, candidate-generator family, or prefix family
being certified is fixed before the certification audit labels are
examined. The audit rule is likewise specified before those labels are
examined. It may adaptively select which item to inspect next and when
to stop as a function of labels revealed during the audit, provided
that its validity has been established under that pre-specified rule.
The audit rule may also depend on information available before
certification, such as the candidate-generator values $z_i=g(x_i)$ and
hence the excluded pool $\{g=0\}$. What is prohibited is the use of
certification labels to construct, tune, order, or select the candidate
generator, candidate-generator family, or prefix family being certified.
\end{assumption}

\begin{assumption}[Noiseless Audit Labels]\label{ass:(A2)}
The audit oracle returns the true relevance label for each audited item, with no labelling noise.
\end{assumption}

Assumption \ref{ass:(A1)} is the validity condition. Each certificate below is derived from a sampling distribution that holds only because the candidate generator, candidate-generator family, or prefix family is fixed before the audit labels are observed. If it were chosen after inspecting those labels, the sampling distribution would no longer apply, and the certificate would no longer be valid.

Candidate-generator design may still be exploratory, heuristic, iterative, or
model-based, provided the certification sample is held out from the design data.

\subsection{Design, certification, and iteration}
\label{sec:design_iteration}

The certificates above apply only to candidate generators, or
candidate-generator families, fixed independently of the certification audit.
This does not prevent exploratory design. The analyst may examine misses, add
rules, tune thresholds, or order components on a design split. The requirement is
that the certification split remains unopened until the candidate-generator
family to be certified has been fixed.

In the default workflow, audit and stop, the design split builds an ordered
sequence \(h_1,\ldots,h_J\), and chooses prefix cut points
\(1\le t_1<\cdots<t_M\le J\). Once the prefix family
\(\{g_1,\ldots,g_M\}\) is frozen, the certification split is opened once.
Corollary~\ref{cor:prefix_stopping} or
Theorem~\ref{thm:shared_reference_prefix} supplies the simultaneous bounds, and
the bicriteria rule of Section~\ref{sec:bicriteria} returns the least burdensome
prefix whose certified missed mass meets the target, or reports that none does.
The certification then ends.

If none passes, the analyst may instead audit and iterate: read the residual
misses in the audit, design a better candidate generator, and certify it in a
new round. The difficulty is that the original audit has now been seen. Its
labels, and even the bare outcome that a prefix passed or failed, carry
information about where relevant items were missed. Any candidate generator
designed in response to that information depends on the opened audit labels and
cannot be certified by them. Those labels may serve only as design data for the
next round. Reusing them to certify a candidate generator they helped design
violates Assumption~\ref{ass:(A1)}, and voids the guarantee.

A valid further round must therefore use fresh certification labels, or a
pre-specified sequential-validity device. The simplest approach is to open a
fresh, previously unused split and certify the new candidate generator on it.
More elaborate approaches, such as pre-specified error
spending~\cite{LanDeMets1983} or anytime-valid confidence
sequences~\cite{HowardEtAl2021}, can also be used, but their validity must be
built into the design before the labels are examined. The common principle is unchanged: labels
used to design a candidate generator cannot also be used to certify it.

Section~\ref{sec:discussion} gives a limited extension of the i.i.d.\ excluded-pool missed-mass certificate to noisy audit labels under a known sensitivity floor. Extending the recall and positive-conditional certificates requires additional assumptions on the label-noise mechanism.

\subsection{Why the audit must look outside the candidate set}
\label{sec:no_inside_only}

Labels from the included pool alone do not constrain the relevance labels in the excluded pool. The next proposition makes this precise.

\begin{proposition}[Included-pool labels alone do not bound missed mass]
\label{obs:no_inside_only}
Fix a candidate generator $g$ with $p_0(g)>0$, and let
\[
r(g)=P(Y=1,g(X)=0)
\]
be the missed relevant mass. Consider any procedure which, using the
candidate-generator values $g(X)$, unlabelled data, and labels observed
only on the included pool $\{g(X)=1\}$, returns an upper bound $T$
for $r(g)$.

Let $\mathbb{P}_{\mathrm A}$ denote probability over the sampling and any
internal randomisation of the procedure, with $P$, $g$, and the
relevance indicator fixed. Suppose that the procedure is valid for every
measurable relevance indicator $\varphi^\star$, in the sense that
\[
\mathbb{P}_{\mathrm{A}}\{r(g)\le T\}\ge 1-\delta.
\]
Then
\[
\mathbb{P}_{\mathrm{A}}\{T\ge p_0(g)\}\ge 1-\delta.
\]
Thus included-pool labels alone cannot certify a non-trivial upper bound
on $r(g)$.
\end{proposition}
\begin{proof}
Let
\[
A=\{x:g(x)=1\},
\qquad
B=\{x:g(x)=0\}.
\]
The procedure observes relevance labels only on $A$. Fix any relevance
indicator $\varphi^\star$, and define a second relevance indicator
$\widetilde\varphi$ which agrees with $\varphi^\star$ on $A$ and
is identically one on $B$.

The procedure receives exactly the same observable information under
$\varphi^\star$ and $\widetilde\varphi$: the candidate-generator
values, the unlabelled data, and all observed included-pool labels are
unchanged. It therefore returns the same random upper bound $T$ in
the two cases.

Under $\widetilde\varphi$, every excluded item is relevant, so the
missed relevant mass is
\[
r_{\widetilde\varphi}(g)
=
P\{\widetilde\varphi(X)=1,g(X)=0\}
=
P\{g(X)=0\}
=
p_0(g).
\]
Applying the assumed validity condition under $\widetilde\varphi$
therefore gives
\[
\mathbb{P}_{\mathrm{A}}\{T\ge p_0(g)\}\ge 1-\delta.
\]
Hence any procedure using only included-pool labels must allow, with
probability at least $1-\delta$, the possibility that the entire
excluded pool is relevant.
\end{proof}
Thus included-pool labels, even combined with the candidate-generator values and unlimited unlabelled data, cannot yield a non-trivial distribution-free certificate for missed relevant mass. The best universal bound remains the excluded-pool mass itself.

\section{Bicriteria formulation}
\label{sec:bicriteria}

The burden term is necessary because the unrestricted one-sided objective is degenerate.

\begin{proposition}[Triviality of unrestricted coverage]
\label{prop:trivial_coverage}
The constant candidate generator $g \equiv 1$ selects every item. Hence $r(g)=0$ for every $P$ and every $\varphi^\star$, and $L(g)=0$ whenever $L$ is defined (that is, whenever $\pi=P(Y=1)>0$). Unrestricted minimisation of $r$, or of $L$, over the space of all candidate generators is therefore degenerate: $g \equiv 1$ is always a minimiser.
\end{proposition}

Excluding this trivial solution requires review burden to enter the formulation. Let $\mathcal G$ be a finite family of candidate generators, fixed before the audit. For each $g \in \mathcal G$, the burden depends only on $g$ and not on the relevance labels. In the finite-corpus model, $B_N(g) = N_1(g)/N$ can be computed exactly from the candidate-generator values. In the population model, $B(g) = P(g(X)=1)$ can be estimated from an unlabelled sample. In either case, no audit labels are needed. The audit returns an upper confidence bound $U(g)$ on the missed relevant mass $r(g)$. For a target $\varepsilon$, consider the constrained problem
\begin{equation}
\label{eq:bicriteria}
\text{minimise } B(g) \text{ subject to } U(g) \le \varepsilon, \quad g \in \mathcal G,
\end{equation}
and report any minimiser $\widehat g$.

When several candidate generators are considered, the bounds are constructed simultaneously, so that
$$
\mathbb{P}_{\mathrm{audit}}\left(
r(g)\le U(g)\ \text{for all }g\in\mathcal G
\right)\ge 1-\delta,
$$
where the probability is over the random audit sample. On this event every candidate generator in $\mathcal G$ satisfies its own bound at once, so a minimiser $\widehat g$ of~\eqref{eq:bicriteria}, selected after the bounds are computed, still satisfies $r(\widehat g)\le U(\widehat g)$ with probability at least $1-\delta$.

The choice between controlling absolute missed relevant mass $r(g)$ and conditional miss rate $L(g)$ is an application decision. Legal discovery may care about the absolute amount of relevant material left outside the candidate set. Systematic-review screening usually phrases the requirement as recall among all relevant items. We support both targets, but the audit designs differ: $r(g)$ can be certified from an excluded-pool audit alone, whereas $L(g)$ also requires information about the positive count in the included pool.

The bicriteria form is a Neyman--Pearson problem in shape \cite{ScottNowak2005,CannonEtAl2002,Tong2013,RigolletTong2011}: one error type is constrained, the other minimised. Classical Neyman--Pearson classification gives sample-complexity bounds for selecting a classifier from a class subject to a one-sided constraint. We use that machinery at the design stage in Section~\ref{sec:learning}. The novel content is on the certification side, where a fixed candidate generator is audited rather than designed (Section~\ref{sec:excluded}). The important separation in~\eqref{eq:bicriteria} is operational: the burden $B(g) = P(g(X) = 1)$ does not require audit labels, while the missed-mass bound $U(g)$ does.

For an ordered union of $J$ component detectors $h_1,\ldots,h_J$, let $g_0 \equiv 0$ and
$$
g_j(x)=\max_{\ell\le j}h_\ell(x), \qquad j=1,\ldots,J.
$$
Define the incremental review burden
$$
\Delta_j^B := P\big(g_{j-1}(X) = 0,\, h_j(X) = 1\big),
$$
the probability that an item is newly caught by $h_j$ but not by any earlier component. The events $\{g_{j-1}=0, h_j=1\}$ are disjoint across $j$, so
\begin{equation}
\label{eq:burden_decomp}
B(g_J) =P\left(g_J(X)=1\right)=\sum_{j=1}^J \Delta_j^B.
\end{equation}

\section{Excluded-pool certification}
\label{sec:excluded}

We now give the central technical results. Proposition~\ref{obs:no_inside_only} shows why included-pool labels cannot certify missed mass. The constructive response is to direct audit labels to the only region where missed relevant items can be found, namely the excluded pool $\{g(X) = 0\}$.

\begin{remark}[Conventions for degenerate cases]
Throughout Sections~\ref{sec:excluded} to~\ref{sec:stress}, conditional distributions are used only when the conditioning event has positive probability, and the degenerate cases are resolved as follows. If $p_0(g)=0$, then $r(g)=0$, no excluded-pool audit is required, and every excluded-pool certificate for $g$ is read as the trivial bound $0$. If $p_1(g)=0$ and $\pi>0$, then $\mathrm{Recall}(g)=0$. Statements involving recall, or sampling from $P(\cdot\mid Y=1)$, assume $\pi>0$. In particular, Sections~\ref{sec:amplification} and~\ref{sec:stress} audit relevant items and assume $\pi>0$ throughout, and the finite-corpus recall statements assume $M_0(g)+M_1(g)>0$, as in Section~\ref{sec:setup}. In the finite-corpus model, audit sizes satisfy $0\le n_0\le N_0(g)$ and $0\le n_1\le N_1(g)$ for every sampled candidate generator. If $N_0(g)=0$, then $M_0(g)=0$ is known without sampling. Empty audits are covered by the conventions $U_0(0,\alpha)=1$ and $L_0(0,\alpha)=0$ of Lemma~\ref{lem:conditioning} below.
\end{remark}

The certificates below repeatedly reduce to the same elementary problem. We observe $K$ relevant items in $n$ audited draws from a pool whose unknown relevance probability is $p$. An upper one-sided confidence bound $U_n(K,\alpha)$ is a number such that, with probability at least $1-\alpha$, the true $p$ is no larger than $U_n(K,\alpha)$. A lower one-sided confidence bound $L_n(K,\alpha)$ is defined analogously. Throughout this section, $0<\alpha<1$ denotes the error probability allocated to one such one-sided bound.

We use the exact Clopper--Pearson bounds, obtained by inverting binomial tail probabilities. For integers $0\le k\le n$ and $\alpha\in(0,1)$, define
\begin{align*}
U_n(k,\alpha)
&:=
\sup\left\{
q\in[0,1]:
\mathbb{P}(\mathrm{Bin}(n,q)\le k)\ge\alpha
\right\},\\
L_n(k,\alpha)
&:=
\inf\left\{
q\in[0,1]:
\mathbb{P}(\mathrm{Bin}(n,q)\ge k)\ge\alpha
\right\}.
\end{align*}
By construction, for $K\sim\mathrm{Bin}(n,p)$,
\[
\mathbb{P}_p\{p>U_n(K,\alpha)\}\le\alpha,
\qquad
\mathbb{P}_p\{p<L_n(K,\alpha)\}\le\alpha,
\]
for every $p\in[0,1]$ \cite{ClopperPearson1934}. Here exact means that the bounds invert exact binomial tail probabilities and carry guaranteed finite-sample coverage. It does not mean that the intervals are shortest or of exactly nominal level: as interval procedures they are conservative, in that the coverage can exceed $1-\alpha$. The same reading of exact applies to the hypergeometric certificate below. The lower-bound symbol $L_n(k,\alpha)$ is unrelated to the conditional miss rate $L(g)$. The subscript and arguments distinguish the two. In the zero-observed-event case, $U_n(0, \alpha) = 1 - \alpha^{1/n}$, for $n\ge 1$. Finding no relevant items in the audited excluded sample does not imply that the excluded pool contains no relevant items. It gives a finite-sample upper bound on the excluded-pool relevance rate.

The next lemma is a conditional form of the usual exact binomial inversion of Clopper and Pearson \cite{ClopperPearson1934}. It will be used when the effective audit size is random, but the relevant count is binomial conditional on that realised size.

\begin{lemma}[Conditional binomial inversion]
\label{lem:conditioning}
Let $Q$ be a non-negative integer-valued random variable and suppose
that, for every $q$ with $\mathbb{P}(Q=q)>0$,
\[
K\mid\{Q=q\}\sim\mathrm{Bin}(q,\theta)
\]
for a fixed $\theta\in[0,1]$. Then, with the convention $U_0(0,\alpha)=1$,
$$
\mathbb{P}\left(\theta>U_Q(K,\alpha)\right)\le\alpha,
$$
so $U_Q(K,\alpha)$ is a valid upper confidence bound for $\theta$ unconditionally. The analogous statement holds for $L_Q(K,\alpha)$, with the convention $L_0(0,\alpha)=0$.
\end{lemma}
\begin{proof}
Let
\[
A:=\{\theta>U_Q(K,\alpha)\}.
\]
By the law of total probability, equivalently the tower property applied
to the indicator $\mathbbm 1_A$,
\begin{align*}
\mathbb{P}(A)
&=
\mathbb E\left[
    \mathbb{P}(A\mid Q)
\right] \\
&=
\sum_{q\ge 0}
\mathbb{P}\left(
  \theta>U_q(K,\alpha)\mid Q=q
\right)
\mathbb{P}(Q=q).
\end{align*}
For every \(q\ge1\) with \(\mathbb{P}(Q=q)>0\), the conditional binomial
assumption and the Clopper--Pearson property give
\[
\mathbb{P}\left(
  \theta>U_q(K,\alpha)\mid Q=q
\right)
\le\alpha.
\]
For \(q=0\), we have \(K=0\) almost surely conditional on \(Q=0\), and
the convention \(U_0(0,\alpha)=1\) gives
\[
\mathbb{P}\left(
  \theta>U_0(K,\alpha)\mid Q=0
\right)
=
0.
\]
Consequently,
\begin{align*}
\mathbb{P}(A)
&\le
\alpha\sum_{q\ge1}\mathbb{P}(Q=q)\\
&=
\alpha\,\mathbb{P}(Q\ge1)
\le\alpha.
\end{align*}
The lower-bound statement follows identically by applying the same
argument to
\[
\{\theta<L_Q(K,\alpha)\},
\]
using the convention \(L_0(0,\alpha)=0\).
\end{proof}

\subsection{The binomial excluded-pool certificate}

Fix a candidate generator $g$ and write $p_0(g) = P(g(X) = 0)$, $\eta(g) = P(Y=1 \mid g(X) = 0)$. The decomposition
$$
r(g) = p_0(g) \cdot \eta(g)
$$
is exact. The quantity $p_0(g)$ is observable without audit labels. The quantity $\eta(g)$ requires labels from the excluded pool.

\begin{theorem}[Excluded-pool binomial certificate]
\label{thm:excluded_binomial}
Let $\mathcal G$ be a finite family of candidate generators, fixed independently of the audit, with $p_0(g)>0$ for every $g \in \mathcal G$. Generators with $p_0(g)=0$ are certified trivially by the conventions above. For each $g \in \mathcal G$, draw a separate sample $X_1, \ldots, X_{n_0}$ i.i.d.\ from the conditional distribution of $X$ given $g(X) = 0$, observe $Y_i = \varphi^\star(X_i)$, and let
$$
K_0(g) := \sum_{i=1}^{n_0} Y_i.
$$
Then, with probability at least $1 - \delta$,
$$
\eta(g) \le U_{n_0}\left(K_0(g), \frac{\delta}{|\mathcal G|}\right) \quad \text{and} \quad r(g) \le p_0(g) \cdot U_{n_0}\left(K_0(g), \frac{\delta}{|\mathcal G|}\right) \quad \text{for all } g \in \mathcal G.
$$
If $n_0\ge 1$ and $K_0(g)=0$,
\begin{equation}
\label{eq:zero_miss_excluded}
r(g) \le p_0(g)\left(1 - (\delta/|\mathcal G|)^{1/n_0}\right) \le \frac{p_0(g) \log(|\mathcal G|/\delta)}{n_0}.
\end{equation}
\end{theorem}

\begin{proof}
For each fixed \(g\in\mathcal G\), because \(g\) and \(\mathcal G\)
are fixed before the certification audit labels are observed, and
because the audit draws i.i.d.\ from
\(P(\,\cdot\mid g(X)=0)\), each audited label satisfies
\[
Y_i\sim\mathrm{Bernoulli}(\eta(g)),
\qquad
\eta(g):=P(Y=1\mid g(X)=0).
\]
Hence
\[
K_0(g)
=
\sum_{i=1}^{n_0}Y_i
\sim
\mathrm{Bin}(n_0,\eta(g)).
\]
By the defining property of \(U_{n_0}\),
\[
\mathbb{P}_{\mathrm{audit}}\left\{
\eta(g)>
U_{n_0}\left(
K_0(g),\frac{\delta}{|\mathcal G|}
\right)
\right\}
\le
\frac{\delta}{|\mathcal G|}.
\]

For each \(g\in\mathcal G\), define the failure event
\[
E_g
:=
\left\{
\eta(g)>
U_{n_0}\left(
K_0(g),\frac{\delta}{|\mathcal G|}
\right)
\right\}.
\]
Then, by the union bound,
\begin{align*}
\mathbb{P}_{\mathrm{audit}}\left\{
\exists g\in\mathcal G:
\eta(g)>
U_{n_0}\left(
K_0(g),\frac{\delta}{|\mathcal G|}
\right)
\right\}
=
\mathbb{P}_{\mathrm{audit}}\left(
\bigcup_{g\in\mathcal G}E_g
\right)\le
\sum_{g\in\mathcal G}
\mathbb{P}_{\mathrm{audit}}(E_g)\le
\sum_{g\in\mathcal G}
\frac{\delta}{|\mathcal G|}=\delta.
\end{align*}
Equivalently, with probability at least \(1-\delta\),
\[
\eta(g)
\le
U_{n_0}\left(
K_0(g),\frac{\delta}{|\mathcal G|}
\right)
\qquad
\text{for every }g\in\mathcal G.
\]
Since $r(g)=p_0(g)\eta(g)$, the same event implies
\[
r(g)
\le
p_0(g)\,
U_{n_0}\left(
K_0(g),\frac{\delta}{|\mathcal G|}
\right)
\qquad
\text{for every }g\in\mathcal G.
\]

The same argument permits generator-specific sample sizes and error
allocations \(\alpha_g\) satisfying
\[
\sum_{g\in\mathcal G}\alpha_g\le\delta.
\]

Finally, if \(n_0\ge 1\) and \(K_0(g)=0\), then, with
$\alpha=\delta/|\mathcal G|$, $U_{n_0}(0,\alpha)=1-\alpha^{1/n_0}$. Therefore
\[
r(g)
\le
p_0(g)
\left[
1-
\left(
\frac{\delta}{|\mathcal G|}
\right)^{1/n_0}
\right].
\]
Writing $u:=\tfrac{\log(|\mathcal G|/\delta)}{n_0}$, we have
\[
1-
\left(
\frac{\delta}{|\mathcal G|}
\right)^{1/n_0}
=
1-e^{-u}
\le u,
\]
and hence $r(g)
\le
\tfrac{
p_0(g)\log(|\mathcal G|/\delta)
}{
n_0
}.
$
\end{proof}
The factor $p_0(g)$ requires no audit labels, because $g$ can be evaluated on any item without knowing its relevance, and both excluded-pool certificates, Theorem~\ref{thm:excluded_binomial} and the finite-corpus Theorem~\ref{thm:excluded_hypergeometric}, leave it outside the audit. Whenever the corpus is finite it is known exactly, as $p_{0,N}(g)=N_0(g)/N$. In the i.i.d.\ population model it is a probability rather than a count, and must be known or estimated. If $p_0(g)$ must instead be estimated, an independent unlabelled sample $X_1^u,\ldots,X_m^u$ gives
$$
\widehat p_0(g)=m^{-1}\sum_{i=1}^m\mathbbm 1\{g(X_i^u)=0\}.
$$
Hoeffding's inequality and a union bound over $\mathcal G$ give
$$
p_0(g)\le
\widehat p_0(g)+
\sqrt{\frac{\log(2|\mathcal G|/\delta)}{2m}}
$$
simultaneously for all $g\in\mathcal G$, with probability at least $1-\delta/2$. Combining this event with the audit event, also taken at error level $\delta/2$, gives the missed-mass certificate with confidence at least $1-\delta$, with $p_0(g)$ replaced by the displayed upper bound.

Unlabelled data do not require relevance judgements, so $m$ can be chosen large enough to make the Hoeffding correction small. The relevance labels in the excluded-pool audit are then used only to bound $\eta(g)=P(Y=1\mid g(X)=0)$. Combining this bound with the upper bound on $p_0(g)$ gives the stated upper bound on $r(g)=p_0(g)\eta(g)$.

\subsection{The hypergeometric finite-corpus certificate}

For finite corpora, the natural distributional assumption is sampling without replacement. For a fixed candidate generator $g$, write $N_0=N_0(g)$ for the excluded-pool size, and let $M_0(g) = |\{i : g(x_i) = 0,\, y_i = 1\}|$ be the unknown number of missed relevant items in that pool. If the audit draws $n_0$ items uniformly without replacement from this pool and observes $K_0$ relevant items, the lower-tail probability $m \mapsto \mathbb{P}(\mathrm{Hypergeom}(N_0, m, n_0) \le k)$ is non-increasing in $m$, so the inversion 
$$
M_U(k, \alpha) := \max\left\{m \in \{0, 1, \ldots, N_0\} : \mathbb{P}\left(\mathrm{Hypergeom}(N_0, m, n_0) \le k\right) \ge \alpha\right\}
$$
gives an upper confidence bound on $M_0(g)$. Where the pool and sample sizes must be displayed, we write $M_U(k,\alpha;N,n)$ for the inversion computed with population size $N$ and sample size $n$. When they are clear from context, we abbreviate to $M_U(k,\alpha)$. This is the exact finite-corpus analogue of the binomial upper confidence bound used in the i.i.d.\ population model. The binomial certificate of Theorem~\ref{thm:excluded_binomial} is exact under sampling with replacement from the excluded pool. For sampling without replacement, the exact finite-corpus certificate is the hypergeometric one given below. In the zero-count case, sampling without replacement is no less favourable than independent sampling for finding at least one relevant item, so the binomial zero-count upper bound remains conservative. Classical comparisons of the hypergeometric and binomial distribution functions are given by Uhlmann~\cite{Uhlmann1966}.

The certificate primarily controls the missed relevant count $M_0(g)$. A recall statement also needs the denominator, namely the total number of relevant items $M_0(g)+M_1(g)$. Thus an excluded-pool audit controls the missed-positive part of recall. The included-pool contribution must be known, lower-bounded, or audited separately.

\begin{theorem}[Excluded-pool hypergeometric certificate]
\label{thm:excluded_hypergeometric}
Let $\mathcal G$ be a finite family of candidate generators, fixed independently of the audit. For each $g \in \mathcal G$, draw $n_0$ items uniformly without replacement from the excluded pool $\{i : g(x_i) = 0\}$, with $0\le n_0\le N_0(g)$, observe their labels, and let $K_0(g)$ be the number of relevant items observed. Pools with $N_0(g)=0$ have $M_0(g)=0$ known without sampling and are certified trivially by the conventions above; when a two-pool sample is drawn, the analogous condition $0\le n_1\le N_1(g)$ is likewise assumed. Then, with probability at least $1 - \delta$,
$$
M_0(g) \le M_U\left(K_0(g), \frac{\delta}{|\mathcal G|};\, N_0(g),\, n_0\right) \quad \text{for all } g \in \mathcal G,
$$
where the inversion for each $g$ is computed with its own excluded-pool size $N_0(g)$.
If, in addition, all items in the included pool $\{i : g(x_i) = 1\}$ are reviewed, so that $M_1(g)$ is known, then
$$
\mathrm{Recall}(g) = \frac{M_1(g)}{M_1(g) + M_0(g)} \ge \frac{M_1(g)}{M_1(g) + M_U\left(K_0(g),\, \delta/|\mathcal G|;\, N_0(g),\, n_0\right)}.
$$
\end{theorem}

\begin{proof}
By Assumption~\ref{ass:(A1)}, $\mathcal G$ is fixed independently of the audit. The sampling design gives $K_0(g) \sim \mathrm{Hypergeom}(N_0(g), M_0(g), n_0)$. Monotonicity of
\[
m\longmapsto
\mathbb{P}\left\{
\mathrm{Hypergeom}(N_0(g),m,n_0)\le k
\right\}
\]
in \(m\) implies
\[
\mathbb{P}_{\mathrm{audit}}\left\{
M_0(g)>
M_U\left(K_0(g),\frac{\delta}{|\mathcal G|}\right)
\right\}
\le
\frac{\delta}{|\mathcal G|}.
\]
 A union bound over $\mathcal G$ gives the simultaneous statement. The recall bound follows from monotonicity of $M_1(g)/(M_1(g) + M_0)$ in $M_0$.
\end{proof}

This is the finite-corpus version of the elusion-sampling argument used in technology-assisted review~\cite{GrossmanCormackGlossary2013}. One samples from the unreviewed, or excluded, pool and uses the number of relevant documents found in that sample to bound the number of relevant documents remaining in the pool. In the zero-count case, the audit coincides with a classical zero-acceptance-number attributes sampling plan with lot size $N_0$ and limiting quality $m/N_0$~\cite{SchillingNeubauer2017}.

\begin{remark}[Recall requires a denominator]
A certificate for absolute missed mass, $r(g)=P(Y=1,g(X)=0)$, does not by itself give a recall certificate unless the total prevalence $\pi=P(Y=1)$ is known or lower-bounded. Since
$$
L(g)=P(g(X)=0\mid Y=1)=\frac{r(g)}{\pi},
$$
an upper bound on $r(g)$ becomes a bound on missed-positive rate only after dividing by a certified lower bound on $\pi$. This distinction is especially important in rare-event settings: a small absolute missed mass may still be non-negligible relative to the total relevant population.
\end{remark}

A missed-mass certificate can be converted into a recall certificate when a lower bound on the total relevant mass is also available. If $r(g)\le u$ and the total relevant mass is at least $a>0$, then
$$
\mathrm{Recall}(g)\ge 1-\frac{u}{a}.
$$
Thus missed-mass certification controls recall only after supplying a denominator bound.

A denominator lower bound may come from a known total relevant mass, a small whole-population audit, or domain knowledge. This gives a recall certificate from an excluded-pool audit, but the bound becomes weaker when the denominator lower bound is small. If the included pool is fully reviewed, the finite-corpus certificate above gives the sharper bound
$$
\frac{M_1(g)}{M_1(g)+M_U},
$$
which does not require a separate denominator floor.

\subsection{Two-pool recall certificate}

Theorem~\ref{thm:excluded_hypergeometric} certifies the number of relevant items missed in the excluded pool. To turn that missed-count certificate into a recall certificate, one also needs information about the number of relevant items recovered in the included pool. If the included pool has been fully reviewed, this number is known. Otherwise, it must be estimated or lower-bounded.

The two-pool design does this by auditing both sides of the candidate-generation boundary. The excluded-pool sample controls the missed positives, while the included-pool sample controls the recovered positives. If the goal is only to certify absolute missed relevant mass, the excluded-pool certificate in Section~\ref{sec:excluded} is sufficient.

\begin{theorem}[Two-pool recall certificate]
\label{thm:two_pool}
Fix a candidate generator $g$ with $\pi>0$, $p_0(g)>0$, and $p_1(g)>0$, and write $p_1(g) = 1 - p_0(g) = P(g(X) = 1)$, $\eta(g) = P(Y=1 \mid g(X)=0)$, and $\eta_1(g) = P(Y=1 \mid g(X)=1)$. By Bayes' theorem, $\eta_1(g)=\mathrm{Recall}(g)\pi/p_1(g)$. Draw, independently of each other, $n_0$ items from the excluded pool $\{g = 0\}$ and $n_1$ items from the included pool $\{g = 1\}$, each sample i.i.d.\ from the corresponding conditional distribution. Let $K_0$ and $K_1$ be the counts of relevant items observed in each audit. Since
$$
P(Y=1,g(X)=1)=p_1(g)\eta_1(g),
\qquad
P(Y=1,g(X)=0)=p_0(g)\eta(g),
$$
the law of total probability gives
$$
P(Y=1)=p_1(g)\eta_1(g)+p_0(g)\eta(g).
$$
Then, with probability at least $1 - \delta_0 - \delta_1$,
$$
\mathrm{Recall}(g) = \frac{p_1(g)\, \eta_1(g)}{p_1(g)\, \eta_1(g) + p_0(g)\, \eta(g)} \ge \frac{p_1(g)\, L_{n_1}(K_1, \delta_1)}{p_1(g)\, L_{n_1}(K_1, \delta_1) + p_0(g)\, U_{n_0}(K_0, \delta_0)}.
$$
\end{theorem}

\begin{proof}
$K_0 \sim \mathrm{Bin}(n_0, \eta(g))$ and $K_1 \sim \mathrm{Bin}(n_1, \eta_1(g))$, independent across pools. The Clopper--Pearson bounds give $\eta(g) \le U_{n_0}(K_0, \delta_0)$ with probability at least $1 - \delta_0$ and $\eta_1(g) \ge L_{n_1}(K_1, \delta_1)$ with probability at least $1 - \delta_1$. A union bound gives the joint event. The recall ratio
$$
\mathrm{Recall}(g) = \frac{p_1 \eta_1}{p_1 \eta_1 + p_0 \eta}
$$
is increasing in $\eta_1$ and decreasing in $\eta$, so substituting the lower bound on $\eta_1$ and the upper bound on $\eta$ gives a lower bound on the recall.
\end{proof}

The factors $p_0(g)$ and $p_1(g)=1-p_0(g)$ do not require relevance judgements. In the finite-corpus model they are known exactly as $N_0(g)/N$ and $N_1(g)/N$. In the population model they may be estimated from an independent unlabelled sample. Specifically, if $\overline p_0(g)$ and $\underline p_1(g)$ are simultaneous upper and lower confidence bounds satisfying $p_0(g)\le \overline p_0(g)$ and $p_1(g)\ge \underline p_1(g)$,
then the recall certificate remains valid after replacing $p_0(g)$ by $\overline p_0(g)$ and $p_1(g)$ by $\underline p_1(g)$:
\[
\operatorname{Recall}(g)\ge
\frac{\underline p_1(g)L_{n_1}(K_1,\delta_1)}{\underline p_1(g)L_{n_1}(K_1,\delta_1)+\overline p_0(g)U_{n_0}(K_0,\delta_0)}.
\]
The confidence levels used to construct these bounds must be included in the overall error allocation. The audit labels are used to bound the two conditional relevance rates $\eta(g)=P(Y=1\mid g(X)=0)$ and
$\eta_1(g)=P(Y=1\mid g(X)=1)$. The allocation of $\delta_0$ and $\delta_1$ need not be equal. When the excluded-pool upper bound is the main source of uncertainty, assigning it a larger share of the error budget can give a sharper recall certificate. The allocation should be fixed before the audit labels are examined.

The same two-pool construction has an exact finite-corpus version.

\begin{corollary}[Finite-corpus two-pool recall certificate]
\label{cor:two_pool_hypergeometric}
In the finite-corpus model, draw $n_0$ items uniformly without replacement from the excluded pool $\{i:z_i=0\}$ and, independently, $n_1$ items uniformly without replacement from the included pool $\{i:z_i=1\}$, observing $K_0$ and $K_1$ relevant items respectively. Let $M_U(\cdot,\cdot)$ be the hypergeometric upper inversion of Theorem~\ref{thm:excluded_hypergeometric} on the excluded pool, computed with population size $N_0$ and sample size $n_0$, and let
$$
M_L(k,\alpha):=\min\left\{m\in\{0,\ldots,N_1\}:\mathbb{P}\left(\mathrm{Hypergeom}(N_1,m,n_1)\ge k\right)\ge\alpha\right\}
$$
be the corresponding lower inversion on the included pool. Then $M_0(g)\le M_U(K_0,\delta_0)$ and $M_1(g)\ge M_L(K_1,\delta_1)$ hold jointly with probability at least $1-\delta_0-\delta_1$. If $M_0(g)+M_1(g)>0$, then on that event
$$
\mathrm{Recall}(g)=\frac{M_1(g)}{M_0(g)+M_1(g)}\ge\frac{M_L(K_1,\delta_1)}{M_L(K_1,\delta_1)+M_U(K_0,\delta_0)},
$$
with the displayed lower bound interpreted as $0$ when $M_L(K_1,\delta_1)=M_U(K_0,\delta_0)=0$.

\end{corollary}

\begin{proof}
The two pools are disjoint and sampled separately, so $K_0\sim\mathrm{Hypergeom}(N_0,M_0(g),n_0)$ and $K_1\sim\mathrm{Hypergeom}(N_1,M_1(g),n_1)$ are independent. Monotonicity of the hypergeometric tail in the success count gives the two one-sided bounds, and a union bound gives the joint event. The ratio $M_1/(M_0+M_1)$ is increasing in $M_1$ and decreasing in $M_0$, so substituting the lower bound on $M_1$ and the upper bound on $M_0$ gives the displayed lower bound on recall.
\end{proof}

\subsection{Whole-population audit as an alternative}

For completeness we record the whole-population certificate, which arises by sampling i.i.d.\ from the full distribution rather than from $\{g(X) = 0\}$.

\begin{theorem}[Whole-population certificate]
\label{thm:whole_population}
Let $\mathcal G$ be a finite family of candidate generators, fixed independently of the audit. Let $(X_i, Y_i)_{i=1}^n$ be an i.i.d.\ audit sample from $P$, and let $K_g = \sum_{i=1}^n \mathbbm 1\{Y_i = 1, g(X_i) = 0\}$. Then, with probability at least $1 - \delta$,
$$
r(g) \le U_n\left(K_g, \frac{\delta}{|\mathcal G|}\right) \quad \text{for all } g \in \mathcal G.
$$
\end{theorem}

\begin{proof}
For each fixed $g\in\mathcal G$, the indicators $\mathbbm 1\{Y_i=1,\, g(X_i)=0\}$ are i.i.d.\ Bernoulli$(r(g))$, so $K_g\sim\mathrm{Bin}(n,r(g))$. The Clopper--Pearson property gives $\mathbb{P}_{\mathrm{audit}}(r(g)>U_n(K_g,\delta/|\mathcal G|))\le\delta/|\mathcal G|$, and a union bound over $\mathcal G$ completes the proof. This is the argument of Theorem~\ref{thm:excluded_binomial} with the sampling design drawing from $P$ rather than from $P(\cdot\mid g(X)=0)$, and with $r(g)$ playing the role of $\eta(g)$.
\end{proof}

\subsection{Optimality of excluded-pool auditing}
\label{sec:lower_bound}

Proposition~\ref{obs:no_inside_only} shows that included-pool labels certify nothing about missed mass. The next result is its quantitative converse: in the finite-corpus model, any valid audit, however adaptive, must label essentially as many excluded-pool items as the designs above. We state the optimality result in the finite-corpus model because this is the setting in which label complexity has its most direct interpretation: the question is how many items from a fixed excluded pool must be labelled to rule out a specified number of missed relevant items. Here an adaptive procedure may choose which item to label next, and when to stop, as a function of the labels already observed. The lower bound allows this flexibility.

The theorem imposes two requirements on the audit procedure. The first is validity: when there are $m$ or more missed relevant items, the procedure should rarely certify fewer than $m$. The second rules out vacuous procedures: when there are no missed relevant items, the procedure should usually be able to certify fewer than $m$.

\begin{theorem}[Lower bound on audit label complexity]
\label{thm:label_lower_bound}
Fix a corpus and a candidate generator $g$ with excluded-pool size $N_0=N_0(g)\ge 1$. Fix an integer $m\in\{1,\ldots,N_0\}$, and let $\delta,\beta\in(0,1)$ satisfy $\beta+\delta<1$. The audit task is to rule out the possibility that the excluded pool contains $m$ or more relevant items, that is, to certify $M_0(g)\le m-1$.

Consider an audit procedure that labels corpus items one at a time,
possibly adaptively and with external randomisation. The audit procedure
is fixed independently of the labelling. At each stage, the next
inspected item, the stopping decision, and the reported bound are
determined by the same pre-specified rule applied to the random seed and
the observed index-label history.

For any fixed labelling y and any realised value of R, the pre-specified audit rule determines the complete audit trajectory recursively from the random seed and the labels revealed along that trajectory. Formally, let $(\Omega_{\mathrm A},\mathcal A_{\mathrm A},\mathbb P_{\mathrm A})$ carry the internal randomness of the procedure and its audit trajectory, and write $\mathbb E_{\mathrm A}$ for expectation under $\mathbb P_{\mathrm A}$. Let $R$ denote the initial random seed, let
$I_t$ be the index inspected at step $t$, and define 
\[
\mathcal F_0:=\sigma(R),
\qquad
\mathcal F_t
:=
\sigma\left(
R,I_1,y_{I_1},\ldots,I_t,y_{I_t}
\right),
\qquad t\geq 1.
\]
The procedure is non-anticipating: on the event that it has not yet
stopped, the next inspected index $I_{t+1}$ is
$\mathcal F_t$-measurable. Its termination time $\tau$ is an almost
surely finite stopping time with respect to
$(\mathcal F_t)_{t\geq 0}$, and the reported upper confidence bound
$\widehat M_U$ is $\mathcal F_\tau$-measurable.

The procedure may inspect items from either the included or excluded
pool. Let
\[
T
:=
\left|
\left\{
I_t:1\leq t\leq\tau,\ g(x_{I_t})=0
\right\}
\right|
\]
be the number of distinct excluded-pool items inspected.

Assume that the procedure is valid at level $\delta$, in the sense that, for every labelling with $M_0(g)\ge m$, it returns an upper bound $\widehat M_U\le m-1$ with probability at most $\delta$. Assume also that the procedure succeeds with high probability when there are no missed relevant items: for every labelling with $M_0(g)=0$,
$$
\mathbb P_{\mathrm A}\{\widehat M_U\le m-1\}\ge 1-\beta.
$$
Then, under any labelling with $M_0(g)=0$,
$$
\mathbb E_{\mathrm A}T\ \ge (1-\beta-\delta)\frac{N_0}{m}.
$$
In particular, if $\beta\le 1/2$ and $\delta<1/4$, then
$$
\mathbb E_{\mathrm A}T\ge \frac{N_0}{4m}.
$$
\end{theorem}

\begin{proof}
In this proof, a labelling $y$ means an assignment of relevance labels to the fixed finite corpus. We compare two possible labellings: a zero-miss labelling $y^0$, under which no excluded-pool item is relevant, and a planted labelling $y^S$, under which a chosen set $S$ of $m$ excluded-pool items is relevant.

Let $y^0$ be a labelling such that $y^0(x)=0$ whenever $g(x)=0$. Thus $M_0(g)=0$ under $y^0$. Let $R$ denote the initial random seed of the audit procedure. For any
fixed labelling $y$ and any realised value of $R$, the non-anticipation
condition determines the complete audit trajectory recursively from
the labels revealed along that trajectory. Under a fixed labelling, $\mathbb P_{\mathrm A}$ and
$\mathbb E_{\mathrm A}$ therefore average over the random seed $R$.
Joint subscripts indicate any additional sources of randomness over
which probability or expectation is taken. For any labelling $y$, write $\widehat M_U(R;y)$ for the upper bound returned by the procedure when it is run with randomness $R$ and true labelling $y$. When the procedure is run under $y^0$ with randomness $R$, let $A(R)$ be the set of distinct excluded-pool items inspected. Thus $T=|A(R)|$. Write
$$
\mathrm{pass}_0(R)=\{\widehat M_U(R;y^0)\le m-1\}.
$$

Choose a uniformly random $m$-subset $S$ of the $N_0$ excluded-pool items, independently of $R$. For each fixed $S$, define $y^S$ to agree with $y^0$ outside $S$ and to label every item in $S$ as relevant. Then $M_0(g)=m$ under $y^S$. Hence the validity assumption gives, for every fixed $S$,
$$
\mathbb P_{\mathrm A}\{\widehat M_U(R;y^S)\le m-1\}\le\delta.
$$
Averaging this inequality over the random choice of $S$ gives
$$
\mathbb P_{\mathrm A,S}\{\widehat M_U(R;y^S)\le m-1\}\le\delta.
$$

Now compare the two runs with the same value of $R$. If $\mathrm{pass}_0(R)$ occurs and
$S\cap A(R)=\varnothing$, then an induction over the audit steps shows
that the runs under $y^0$ and $y^S$, coupled using the same random seed
$R$, have identical trajectories. At each step, the two runs have observed the same history, so the same
pre-specified non-anticipating audit rule makes them choose the same
next item. Every inspected excluded-pool item lies outside $S$, and therefore has
the same label under $y^0$ and $y^S$. Consequently, the two runs stop
at the same time and return the same value. Hence
$$
\widehat M_U(R;y^S)=\widehat M_U(R;y^0)\le m-1.
$$
Consequently, for every pair $(R,S)$,
$$
\mathbbm 1\{\widehat M_U(R;y^S)\le m-1\}
\ge
\mathbbm 1\{\mathrm{pass}_0(R)\}\mathbbm 1\{S\cap A(R)=\varnothing\}.
$$
Taking expectation over $R$ and $S$, and using the preceding averaged validity bound, gives
$$
\delta\ge
\mathbb E_{\mathrm A,S}
\left[\mathbbm 1\{\mathrm{pass}_0(R)\}\mathbbm 1\{S\cap A(R)=\varnothing\}\right].
$$
Averaging first over $S$, conditional on $R$, gives
$$
\begin{aligned}
&\mathbb E_{\mathrm A,S}\left[
\mathbbm 1\{\mathrm{pass}_0(R)\}\mathbbm 1\{S\cap A(R)=\varnothing\}
\right] \\
&\qquad =
\mathbb E_{\mathrm A}\left[
\mathbbm 1\{\mathrm{pass}_0(R)\}
\mathbb P_S\{S\cap A(R)=\varnothing\mid A(R)\}
\right].
\end{aligned}
$$

For a realised inspected set $A(R)$, each excluded-pool item belongs to $S$ with probability $m/N_0$. A union bound gives
$$
\mathbb P_S\{S\cap A(R)\ne\varnothing\mid A(R)\}\le \frac{m|A(R)|}{N_0}=\frac{mT}{N_0}.
$$
Therefore
$$
\mathbb P_S\{S\cap A(R)=\varnothing\mid A(R)\}\ge 1-\frac{mT}{N_0}.
$$
For each realised run under $y^0$,
$$
\mathbbm 1\{\mathrm{pass}_0(R)\}\mathbb P_S\{S\cap A(R)=\varnothing\mid A(R)\}
\ge
\mathbbm 1\{\mathrm{pass}_0(R)\}\left(1-\frac{mT}{N_0}\right)
\ge
\mathbbm 1\{\mathrm{pass}_0(R)\}-\frac{mT}{N_0}.
$$
Taking expectations and using the assumption that the procedure certifies correctly under every zero-miss labelling gives
$$
\delta\ge
\mathbb P_{\mathrm A}\{\mathrm{pass}_0(R)\}-\frac{m}{N_0}\mathbb E_{\mathrm A} T
\ge
(1-\beta)-\frac{m}{N_0}\mathbb E_{\mathrm A} T.
$$
Rearranging gives the claim.
\end{proof}

\begin{remark}[Comparison with the simple zero-count audit]
\label{rem:zero_count_comparison}
The following comparison shows that the lower bound has the right dependence on $N_0$ and $m$ in the favourable case $M_0(g)=0$, where there are no missed relevant items in the excluded pool. It is not intended as an exact optimality statement for all possible audit outcomes.

Consider the standard non-adaptive audit which samples $n_0$ items uniformly without replacement from the excluded pool. If the excluded pool contains at least $m$ relevant items, then the probability that the audit sees no relevant item is at most
$$
\frac{\binom{N_0-m}{n_0}}{\binom{N_0}{n_0}}
\le
\left(1-\frac{n_0}{N_0}\right)^m .
$$
Thus, if
$$
n_0\ge N_0(1-\delta^{1/m}),
$$
then a zero-count outcome certifies $M_0(g)\le m-1$ with error probability at most $\delta$.

This sufficient audit size has the same dependence on $N_0$ and $m$ as the lower bound. Indeed,
$$
N_0(1-\delta^{1/m})
\le
\frac{N_0}{m}\log(1/\delta).
$$
Theorem~\ref{thm:label_lower_bound} shows that any valid procedure which usually certifies $M_0(g)\le m-1$ when $M_0(g)=0$ must inspect at least
$$
(1-\beta-\delta)\frac{N_0}{m}
$$
excluded-pool items in expectation under the zero-miss labelling. Hence the simple zero-count excluded-pool audit matches the lower-bound scale up to the constant factor
$$
\frac{\log(1/\delta)}{1-\beta-\delta},
$$
which does not depend on $N_0$ or $m$.
\end{remark}

\begin{corollary}[Minimax label complexity of missed-mass certification]
\label{cor:minimax}
Fix the excluded-pool size $N_0$ and an integer $m\in\{1,\ldots,N_0\}$, and fix error parameters $\beta\le 1/2$, $\delta<1/4$. Consider the zero-miss regime $M_0(g)=0$. Among all audit procedures that are valid at level $\delta$ and certify $M_0(g)\le m-1$ with probability at least $1-\beta$ whenever $M_0(g)=0$, the minimum expected number of excluded-pool labels inspected under the zero-miss labelling is
$$
\Theta\left(\frac{N_0}{m}\right),
$$
where the implied constants depend only on $(\beta,\delta)$ and not on $N_0$ or $m$. The lower bound is Theorem~\ref{thm:label_lower_bound}, whereas the upper bound is attained by the non-adaptive zero-count audit of the preceding remark, which inspects at most $(N_0/m)\log(1/\delta)+1$ labels.
\end{corollary}

\begin{proof}
Let $\mathfrak A_{\beta,\delta}$ denote the class of audit procedures satisfying the validity and zero-miss success requirements in the statement. By Theorem~\ref{thm:label_lower_bound}, every $\mathcal A\in\mathfrak A_{\beta,\delta}$ satisfies
\[
\mathbb E_{\mathrm A}T\ge(1-\beta-\delta)\frac{N_0}{m}.
\]
Under the assumed restrictions $\beta\le 1/2$ and $\delta<1/4$, this gives
\[
\inf_{\mathcal A\in\mathfrak A_{\beta,\delta}}\mathbb E_{\mathrm A}T\ge\frac{N_0}{4m}.
\]

For the matching upper bound, consider the non-adaptive audit of Remark~\ref{rem:zero_count_comparison}, which samples
\[
n_0=\left\lceil N_0\left(1-\delta^{1/m}\right)\right\rceil
\]
items uniformly without replacement from the excluded pool and
certifies $M_0(g)\le m-1$ if no relevant item is observed. If
$M_0(g)\ge m$, the probability of observing no relevant item is at
most $\delta$, so the procedure is valid at level $\delta$. If
$M_0(g)=0$, it certifies with probability one, and hence with
probability at least $1-\beta$.

Moreover, using $1-e^{-u}\le u$ with $u=\log(1/\delta)/m$,
\[
n_0\le\frac{N_0}{m}\log\frac{1}{\delta}+1.
\]
Since $m\le N_0$, the additive term is itself at most $N_0/m$, and therefore
\[
\inf_{\mathcal A\in\mathfrak A_{\beta,\delta}}
\mathbb E_{\mathrm A}T
\le
\left(1+\log\frac{1}{\delta}\right)\frac{N_0}{m}.
\]
Combining the lower and upper bounds proves
\[
\inf_{\mathcal A\in\mathfrak A_{\beta,\delta}}
\mathbb E_{\mathrm A}T
=
\Theta\left(\frac{N_0}{m}\right),
\]
where the implied constants depend only on $(\beta,\delta)$.
\end{proof}

\subsection{Sample-size planning}
\label{sec:samplesize}

The zero-observed-miss case of Theorem~\ref{thm:excluded_binomial} translates targets directly into audit sizes. In the planning calculations below, $\delta$ denotes the total audit error budget. The zero-count identity is the exact form of the rule of three, namely $n\approx 3/\varepsilon$ at $\delta=0.05$~\cite{HanleyLippmanHand1983,TuylGerlachMengersen2009}. To certify $\eta(g) \le \varepsilon$ with confidence $1 - \delta$ from zero observed misses on a single candidate generator,
$$
n_0 \ge \frac{\log(1/\delta)}{\log(1/(1 - \varepsilon))}.
$$
For $\delta = 0.05$, the implied audit sizes are shown below. In the zero-count case, the required audit size is obtained from $(1-\varepsilon)^n\le\delta$, so $n$ is proportional to $\log(1/\delta)$ for fixed $\varepsilon$. To certify $M$ pre-specified prefixes simultaneously, we allocate error level $\delta/M$ to each prefix. This replaces $\log(1/\delta)$ by $\log(M/\delta)=\log(1/\delta)+\log M$, so the required number of labels
per prefix increases only logarithmically in $M$. With separate excluded-pool audits, the total label cost is $M$ times the displayed per-prefix figure. By contrast, the shared-reference design of Section~\ref{sec:prefix} reuses one labelled reference sample across all prefixes. The third column is for an audit plan that certifies ten pre-specified candidate generators or prefixes using a union bound, so each individual bound receives error budget $\delta/10=0.005$.

\begin{center}
\renewcommand{\arraystretch}{1.2}
\begin{tabular}{rcc}
\hline
Target $\varepsilon$ on $\eta(g)$
&
Single candidate generator
&
\shortstack{10 pre-specified prefixes\\(labels per prefix)}
\\
\hline
$0.10$ & $29$ & $51$ \\
$0.05$ & $59$ & $104$ \\
$0.02$ & $149$ & $263$ \\
$0.01$ & $299$ & $528$ \\
$0.005$ & $598$ & $1{,}058$ \\
\hline
\end{tabular}
\end{center}

To plan for an absolute missed-mass target
$r(g)\le\varepsilon_r$ when $0<\varepsilon_r<p_0(g)$, use the identity $r(g)=p_0(g)\eta(g)$ and set the excluded-pool target to $\eta(g)\le\frac{\varepsilon_r}{p_0(g)}$. If $\varepsilon_r\ge p_0(g)$, the target already holds without an audit, since $r(g)\le p_0(g)$.
For comparison, if the target is $r(g)\le \varepsilon_r$, a zero-event whole-population audit requires
$$
n \ge \frac{\log(1/\delta)}{\log(1/(1-\varepsilon_r))},
$$
whereas a zero-event excluded-pool audit requires
$$
n_0 \ge \frac{\log(1/\delta)}{\log(1/(1-\varepsilon_r/p_0(g)))}.
$$
For certifying missed relevant mass for a fixed candidate generator, whole-population auditing is mainly a baseline. It is less label-efficient than excluded-pool auditing because many sampled items fall outside the only region where missed relevant items can occur. Whole-population auditing is useful for different targets, such as estimating overall prevalence or supplying a denominator for recall.

\subsection{Certifying and choosing among pre-specified prefixes}
\label{sec:prefix}

We now address a common operational choice. A candidate generator is built by adding component detectors $h_1,\ldots,h_J$ in a fixed order, and the analyst wants to choose how far down the list to go. Each prefix length $t\in\{1,\ldots,J\}$ defines a prefix generator
$$
g_t(x):=\max_{j\le t}h_j(x),
$$
with larger $t$ usually trading higher review burden against smaller missed relevant mass. The audit goal is to certify a pre-specified collection of these prefixes simultaneously, then select the least burdensome one whose certified missed mass is below the target.

The collection certified need not include every prefix. Simultaneous certification pays a union-bound penalty in the number $M$ of prefixes certified, for example through a per-prefix allocation $\delta/M$ in Corollary~\ref{cor:prefix_stopping} and Theorem~\ref{thm:shared_reference_prefix}. Hence the analyst may choose $M\ll J$ cut points
$$
1\le t_1<t_2<\cdots<t_M\le J
$$
in advance, and certify the corresponding sub-sequence
$$
g_m(x):=g_{t_m}(x)=\max_{j\le t_m}h_j(x),\qquad m=1,\ldots,M.
$$
The case $M=J$, $t_m=m$, recovers certification of every prefix. The prefix-specific excluded pool is
$$
E_m:=\{x:g_m(x)=0\}.
$$
By construction the prefix unions are nested:
$$
\widehat S_1\subseteq\widehat S_2\subseteq\cdots\subseteq\widehat S_M,
\qquad
E_1\supseteq E_2\supseteq\cdots\supseteq E_M.
$$
Adding component detectors can only enlarge the candidate set, and hence can only shrink the excluded pool. Two sampling designs are useful.

\begin{corollary}[Separate excluded-pool prefix certificate]
\label{cor:prefix_stopping}
Suppose $h_1,\ldots,h_J$ and $t_1,\ldots,t_M$ are fixed independently of the audit. For each $m$, draw $n_m$ samples i.i.d.\ from $P(\cdot \mid g_m(X)=0)$, observe their labels, and let $K_m$ be the number of relevant items observed in the $m$th excluded-pool audit. If $\delta_1,\ldots,\delta_M>0$ satisfy $\sum_m \delta_m \le \delta$, then, with probability at least $1-\delta$,
$$
r(g_m) \le p_0(g_m) \cdot U_{n_m}(K_m, \delta_m), \qquad m=1,\ldots,M.
$$
In particular, taking $n_m=n$ for all $m$ and $\delta_m=\delta/M$ gives simultaneous prefix certificates with audit size $n$ per prefix. Sample sizes achieving a target $\varepsilon_r$ per prefix follow from Section~\ref{sec:samplesize}, with the per-prefix error level set to $\delta/M$.
\end{corollary}

\begin{proof}
Apply Theorem~\ref{thm:excluded_binomial} to the finite family $\{g_1,\ldots,g_M\}$, allowing unequal sample sizes and confidence allocations. A union bound over $m$ gives the result.
\end{proof}

The labelling cost depends on the sampling design. Corollary~\ref{cor:prefix_stopping} uses separate excluded-pool audits, so it is cleanest theoretically but may require repeated sampling from the shrinking sequence of excluded pools $E_1\supseteq\cdots\supseteq E_M$. A single held-out certification sample requires a different argument.

\begin{theorem}[Shared-reference prefix certificate]
\label{thm:shared_reference_prefix}
Suppose $h_1,\ldots,h_J$ and $t_1,\ldots,t_M$ are fixed independently of the audit, and let $E_m=\{x:g_m(x)=0\}$. Let $\mathcal R$ be a fixed measurable reference pool with $P(\mathcal R)>0$ and $E_m\subseteq\mathcal R$ for every $m$. Examples are $\mathcal R=\mathcal X$ and $\mathcal R=E_1$, the latter exploiting that $E_m\subseteq E_1$ for every $m$. Draw $Z_1,\ldots,Z_Q$ i.i.d.\ from $P(\cdot\mid\mathcal R)$ and label all sampled items. Define
$$
n_m^{\mathrm{eff}} := \sum_{i=1}^Q \mathbbm 1\{Z_i\in E_m\}, \qquad K_m := \sum_{i=1}^Q \mathbbm 1\{Z_i\in E_m,\, \varphi^\star(Z_i)=1\}.
$$
With the convention $U_0(0,\alpha)=1$, with probability at least $1-\delta$,
$$
r(g_m) \le p_0(g_m) \cdot U_{n_m^{\mathrm{eff}}}\left(K_m,\frac{\delta}{M}\right), \qquad m=1,\ldots,M.
$$
Any stopping rule that selects a prefix using only these simultaneous certificates has confidence at least $1-\delta$.
\end{theorem}

The certificate is computed from the realised effective sample sizes $n_m^{\mathrm{eff}}$, but its validity is unconditional. The proof uses only a conditioning argument: among the reference samples that fall in a given excluded pool, the labels behave as a binomial sample from that excluded pool.

\begin{proof}
Fix $m$. If $P(E_m)=0$, then $p_0(g_m)=0$, and the displayed bound is trivial. Assume $P(E_m)>0$.

The reference sample is drawn from $P(\cdot\mid \mathcal R)$, and $E_m\subseteq\mathcal R$. Conditional on the event $Z_i\in E_m$, the distribution of $Z_i$ is therefore $P(\cdot\mid E_m)$, which is the same as $P(\cdot\mid g_m(X)=0)$. Hence a sampled point that falls in $E_m$ is relevant with probability
$$
\eta(g_m)=P(Y=1\mid g_m(X)=0).
$$

Now condition on the realised effective sample size $n_m^{\mathrm{eff}}=q$. Given this event, the $q$ sampled points that fall in $E_m$ are distributed as $q$ independent draws from $P(\cdot\mid E_m)$. Therefore
$$
K_m\mid \{n_m^{\mathrm{eff}}=q\}\sim
\mathrm{Bin}(q,\eta(g_m)).
$$
By Lemma~\ref{lem:conditioning},
$$
\mathbb{P}_{\mathrm{audit}}\left\{
\eta(g_m)>
U_{n_m^{\mathrm{eff}}}\left(K_m,\frac{\delta}{M}\right)
\right\}
\le \frac{\delta}{M}.
$$
Multiplying by $p_0(g_m)$ gives
$$
r(g_m)=p_0(g_m)\eta(g_m)
\le
p_0(g_m)
U_{n_m^{\mathrm{eff}}}\left(K_m,\frac{\delta}{M}\right)
$$
with probability at least $1-\delta/M$.

A union bound over $m=1,\ldots,M$ gives the simultaneous statement with probability at least $1-\delta$. On that event, every displayed bound holds at once, so any stopping rule that selects a prefix using only these certificates selects a prefix whose corresponding bound is valid.
\end{proof}

The shared-reference design lets the analyst open one held-out certification sample and certify all pre-specified prefixes. Its cost is that the effective sample size is random. Since
$$
E_M\subseteq\cdots\subseteq E_1,
$$
we have
$$
n_M^{\mathrm{eff}}\le\cdots\le n_1^{\mathrm{eff}}.
$$
Thus later prefixes, which usually have smaller excluded pools, may have the least precise certificates. The factor $p_0(g_m)$ in the displayed bound is known exactly in the finite-corpus model, as $N_0(g_m)/N$. In the population model, it can be replaced by the Hoeffding upper bound following Theorem~\ref{thm:excluded_binomial}, computed from an independent unlabelled sample, with the corresponding confidence adjustment.

In planning such an audit, the analyst should budget the total reference sample $Q$ so that the smallest effective sample size is large enough. Under the prefix nesting above, this means controlling $n_M^{\mathrm{eff}}$. The finite-corpus version follows by the same conditioning argument.

\begin{corollary}[Finite-corpus shared-reference prefix certificate]
\label{cor:shared_reference_fc}
In the finite-corpus model, suppose $h_1,\ldots,h_J$ and $t_1,\ldots,t_M$ are fixed independently of the audit, and let $E_m=\{i:g_m(x_i)=0\}$. Let $\mathcal R\subseteq\{1,\ldots,N\}$ be a fixed reference pool with $E_m\subseteq\mathcal R$ for every $m$. Draw $Q$ items uniformly without replacement from $\mathcal R$, where $0\le Q\le |\mathcal R|$, and label them. Let $n_m^{\mathrm{eff}}$ be the number of sampled items in $E_m$ and $K_m$ the number of those that are relevant. Then, with probability at least $1-\delta$,
$$
M_0(g_m)\le M_U\left(K_m,\frac{\delta}{M};\,N_0(g_m),\,n_m^{\mathrm{eff}}\right),
\qquad m=1,\ldots,M.
$$
\end{corollary}

\begin{proof}
Fix $m$. Conditional on $\{n_m^{\mathrm{eff}}=q\}$, the $q$ sampled items falling in $E_m$ are, by exchangeability of the uniform reference draw, a uniform without-replacement sample from $E_m$, so
$$
K_m\mid\{n_m^{\mathrm{eff}}=q\}\sim\mathrm{Hypergeom}\big(N_0(g_m),M_0(g_m),q\big).
$$
For each $q\ge 1$, the inversion is valid at level $\delta/M$ conditionally, by the argument of Theorem~\ref{thm:excluded_hypergeometric}, and for $q=0$ it returns the trivial bound $N_0(g_m)$. Averaging over the distribution of $n_m^{\mathrm{eff}}$, as in Lemma~\ref{lem:conditioning}, gives unconditional validity, and a union bound over $m=1,\ldots,M$ completes the proof.
\end{proof}

\subsection{Fixed-target stopping without a multiplicity penalty}
\label{sec:fixed_sequence}

The simultaneous certificates of Corollary~\ref{cor:prefix_stopping} and Theorem~\ref{thm:shared_reference_prefix} deliver numerical bounds for all $M$ prefixes at a $\delta/M$ union-bound cost, and remain valid against a target chosen after the audit. When the missed-mass target $\varepsilon$ is itself fixed before the audit, as in the bicriteria rule of Section~\ref{sec:bicriteria}, that penalty can be removed.

\begin{theorem}[Fixed-sequence prefix certification]
\label{thm:fixed_sequence}
Suppose that the component detectors $h_1,\ldots,h_J$, the cut points
$$
1\le t_1<\cdots<t_M\le J,
$$
and the missed-mass target $\varepsilon>0$ are fixed before the audit labels are observed. Let $g_1,\ldots,g_M$ be the corresponding prefix generators. Thus
$$
E_M\subseteq \cdots \subseteq E_1,
\qquad
r(g_1)\ge \cdots \ge r(g_M).
$$

For each $m$, let $\mathrm{pass}_m$ be the event that a valid level-$\delta$ excluded-pool certificate for $g_m$ certifies
$$
r(g_m)\le \varepsilon.
$$
Equivalently, whenever $r(g_m)>\varepsilon$,
$$
\mathbb{P}_{\mathrm{audit}}\{\mathrm{pass}_m\}\le\delta.
$$

Test the prefixes in the fixed order
$$
g_M,\ g_{M-1},\ \ldots,\ g_1.
$$
If $\mathrm{pass}_M$ fails, return no prefix. Otherwise, continue down the sequence until the first failure. Return the last prefix that passed. If all tests pass, return $g_1$.

Then, with probability at least $1-\delta$, either no prefix is returned, or the returned prefix satisfies
$$
r(g)\le \varepsilon.
$$
\end{theorem}

\begin{proof}
Suppose first that $r(g_M)>\varepsilon$. By monotonicity, every prefix then has missed mass greater than $\varepsilon$. The procedure returns a prefix only on the event $\mathrm{pass}_M$, and validity gives
$$
\mathbb{P}_{\mathrm{audit}}\{\mathrm{pass}_M\}\le\delta.
$$
Hence, with probability at least $1-\delta$, no prefix is returned.

Otherwise, let
$$
m_0:=\min\{m:r(g_m)\le\varepsilon\}.
$$
By monotonicity, $r(g_m)\le\varepsilon$ for all $m\ge m_0$. If $m_0=1$, every prefix meets the target, so any returned prefix is valid, and the conclusion holds with probability one. If $m_0>1$, then $r(g_{m_0-1})>\varepsilon$, and the tests proceed through $g_M,\ldots,g_{m_0}$ before reaching $g_{m_0-1}$. Since the procedure stops at the first failure and returns the last prefix that passed, it returns a prefix with index below $m_0$ only if the test at $g_{m_0-1}$ is reached and passes. Hence
$$
\mathbb{P}_{\mathrm{audit}}\{\text{the returned prefix is invalid}\}
\le
\mathbb{P}_{\mathrm{audit}}\{\mathrm{pass}_{m_0-1}\}
\le\delta.
$$
The final inequality follows from validity at $g_{m_0-1}$. In every case, with probability at least $1-\delta$, either no prefix is returned or the returned prefix satisfies $r(g)\le\varepsilon$.
\end{proof}
The argument uses only the marginal level of the single boundary test, so the prefix audits may share data, as in the shared-reference design, without changing the validity argument.

\begin{remark}[Fixed-sequence stopping versus optional stopping]
The stopping in Theorem~\ref{thm:fixed_sequence} is stopping along a
finite, pre-specified sequence of tests. It is not optional stopping of
a confidence procedure at a data-dependent audit sample size. The proof
uses only the marginal level-$\delta$ validity of the single boundary
test and therefore requires no martingale or optional-stopping theorem.

If the sample size or sampling distribution used to construct an
individual certificate is itself chosen adaptively after examining
audit labels, its level-$\delta$ validity must be established under that
adaptive sampling rule. This may follow from a conditional argument,
as in Lemma~\ref{lem:conditioning} when the required conditional
binomial law holds, or from an anytime-valid procedure.
\end{remark}

\begin{remark}[What is gained and what is lost]
\label{rem:fs_tradeoff}
The gain is that each test runs at level $\delta$, rather than using a $\delta/M$ allocation across all prefixes. For a pre-registered missed-mass target, this can reduce the audit size needed at each tested prefix. The cost is that the procedure certifies only the selected prefix against the fixed target. It does not deliver simultaneous numerical bounds for every prefix, and the pass/fail decisions at $\varepsilon$ cannot be reused for a different target chosen later. When post-hoc numerical bounds or target flexibility are wanted, the simultaneous certificates remain the appropriate tool.
\end{remark}
\begin{remark}[Provenance]
\label{rem:fs_provenance}
The fixed-sequence testing principle is standard in multiple testing. It appears, for example, in ordered and stepwise testing procedures~\cite{Bauer1991,MaurerHothornLehmacher1995,HsuBerger1999}, and is also used in Learn then Test~\cite{LTT2021}. We use it here in a different setting: the ordered hypotheses are missed-mass certificates for nested prefix candidate generators.
\end{remark}

\section{Coverage amplification with auditable diagnostics}
\label{sec:amplification}

The excluded-pool certificate gives an upper bound on missed mass from the audit labels, but it does not explain how the components of a candidate-generator union contribute to that bound. We now give a structural result that decomposes the conditional miss rate $L(g_J)$ across an ordered component sequence $h_1, \ldots, h_J$, with a diagnostic interpretation of each component's contribution.

\subsection{Incremental coverage and exponential decay}

\begin{definition}[Incremental coverage]
\label{def:incremental}
Let $h_1, \ldots, h_J$ be an ordered sequence of component detectors and let $g_j = \max_{\ell \le j} h_\ell$ with $g_0 \equiv 0$. The \emph{incremental coverage} at round $j$ is
$$
\gamma_j := P\big(h_j(X) = 1 \mid g_{j-1}(X) = 0,\, Y = 1\big),
$$
the conditional probability that component detector $h_j$ catches a relevant item not yet caught by the previous detectors. If the conditioning event has zero probability, set $\gamma_j = 1$.
\end{definition}

To see how the incremental coverage probabilities determine the miss rate, first consider two component detectors. Since the union $g_2$ misses an item exactly when both $h_1$ and $h_2$ miss it,
$$
\{g_2(X)=0\}
=
\{h_1(X)=0,\ h_2(X)=0\}.
$$
Therefore, by the conditional-probability chain rule,
\begin{align*}
L(g_2)
&=
P\big(h_1(X)=0,\ h_2(X)=0\mid Y=1\big)\\
&=
P\big(h_1(X)=0\mid Y=1\big)\,
P\big(h_2(X)=0\mid h_1(X)=0,\ Y=1\big)\\
&=
(1-\gamma_1)(1-\gamma_2).
\end{align*}

More generally,
$$
\{g_J(X)=0\}
=
\bigcap_{j=1}^J\{h_j(X)=0\},
\quad
\mathrm{and}
\quad
\{g_{j-1}(X)=0\}
=
\bigcap_{\ell=1}^{j-1}\{h_\ell(X)=0\}.
$$
Repeated application of the conditional-probability chain rule therefore gives the exact identity
\begin{align*}
L(g_J)
&=
P(g_J(X)=0\mid Y=1)\\
&=
\prod_{j=1}^J
P\big(h_j(X)=0\mid g_{j-1}(X)=0,\ Y=1\big)\\
&=
\prod_{j=1}^J(1-\gamma_j).
\end{align*}
If $P(g_{j-1}(X)=0,Y=1)=0$ for some $j$, then the previous union already misses no relevant items, so all subsequent miss rates are zero; the convention $\gamma_j=1$ preserves the identity.

Hence, using $1-u\le \exp(-u)$ for $u\in[0,1]$,
\begin{equation}
\label{eq:amplification_identity}
L(g_J)
=
\prod_{j=1}^J (1 - \gamma_j)
\le
\exp\left(-\sum_{j=1}^J \gamma_j\right).
\end{equation}
This has the same informal flavour as boosting analyses~\cite{FreundSchapire1997}: later components are valuable only through the mass they recover among items not already captured by earlier components. The analogy is only structural, however. Here the union is a fixed OR-combination of candidate generators, and the identity is an exact coverage decomposition rather than a reweighting or margin argument.

Identity~\eqref{eq:amplification_identity} is not directly auditable because the $\gamma_j$ are unobservable. The next theorem replaces them by Clopper--Pearson lower bounds computed on audited positives.

\subsection{Auditing incremental coverage}

The sharp overall bound on the conditional miss rate $L(g_J)$ is the direct excluded-pool-style certificate of Proposition~\ref{prop:direct_miss_rate} below, which pays a single confidence allocation and is what we recommend reporting. The purpose of this section is different and complementary: to \emph{attribute} coverage to the individual components of an ordered union, so that an analyst can see how much each
added detector contributes among items the earlier detectors still miss. The decomposition that follows is a diagnostic; as a bound on $L(g_J)$ it is deliberately looser than Proposition~\ref{prop:direct_miss_rate}, for the reasons explained in Section~\ref{sec:residual_shrinkage_label}.

We now certify the per-component contributions of $h_1,\ldots,h_J$ to the coverage of the union $g_J$, rather than the coverage of the union as a whole. The question is how much each added component covers among relevant items that the earlier components still miss. The audited sample therefore consists of relevant items drawn according to $P(\cdot\mid Y=1)$, rather than from the positives associated with any particular candidate generator or component. For round $j$, $Q_j$ is the number of audited positives not caught by the previous union $g_{j-1}$, and $H_j$ is the number of those residual positives caught by the next component $h_j$. A lower confidence bound on $H_j/Q_j$ gives a lower bound on the incremental coverage $\gamma_j$. Multiplying the resulting residual factors gives an upper bound on the conditional miss rate of the full union. Section~\ref{sec:residual_shrinkage_label} explains why this product bound is looser than the direct miss-rate certificate on $L(g_J)$, and why the decomposition is included for its per-component content rather than as the reported overall guarantee.

\begin{theorem}[Auditable amplification]
\label{thm:auditable_amplification}
Suppose \(h_1,\ldots,h_J\) and their order are fixed independently of the
audit. Let \(Q\) be the number of audited positives. Conditional on \(Q\),
assume that \(X_1^+,\ldots,X_Q^+\) are independent draws from
\(P(\cdot\mid Y=1)\). This includes, for example, a whole-population audit in
which items are drawn from \(P\), labelled, and the positives are retained.

For \(j=1,\ldots,J\), let \(Q_j\) be the number of audited positives not caught
by \(g_{j-1}\), and let \(H_j\) be the number of these residual audited
positives caught by \(h_j\). Let \(\delta_1,\ldots,\delta_J>0\) satisfy
\[
\sum_{j=1}^J \delta_j\le \delta .
\]
With the convention \(L_0(0,\alpha)=0\), define
\[
\widehat\ell_j=L_{Q_j}(H_j,\delta_j),
\qquad j=1,\ldots,J .
\]
Then, with probability at least \(1-\delta\),
\[
\gamma_j\ge \widehat\ell_j
\quad\text{for all }j=1,\ldots,J .
\]
Consequently,
\[
L(g_J)\le \prod_{j=1}^J(1-\widehat\ell_j).
\]
\end{theorem}

\begin{proof}
Fix $j$ and assume $P(Y=1,\, g_{j-1}(X)=0)>0$. Given $\{Q_j=q\}$ with $q\ge 1$, the $q$ audited positives with $g_{j-1}(X_i^+)=0$ are i.i.d.\ from $P(\cdot\mid Y=1,\, g_{j-1}(X)=0)$, and each is independently caught by $h_j$ with probability $\gamma_j$ by Definition~\ref{def:incremental}. Hence, conditional on $\{Q_j=q\}$, $H_j\sim\mathrm{Bin}(q,\gamma_j)$, and Lemma~\ref{lem:conditioning}, in its lower-bound form with the convention $L_0(0,\delta_j)=0$, gives
$$
\mathbb{P}_{\mathrm{audit}}\left\{
\gamma_j<L_{Q_j}(H_j,\delta_j)
\right\}
\le\delta_j.
$$
If instead $P(Y=1,\, g_{j-1}(X)=0)=0$, then $\gamma_j=1$ by convention and the lower bound is vacuously valid, so the displayed inequality holds for every $j$.

A union bound over $j=1,\ldots,J$ with $\sum_{j=1}^J \delta_j\le\delta$ gives the simultaneous lower bound on the $\gamma_j$. On the event where every $\gamma_j\ge L_{Q_j}(H_j,\delta_j)$, monotonicity in each factor gives
$$
\prod_{j=1}^J (1-\gamma_j) \le \prod_{j=1}^J \big(1 - L_{Q_j}(H_j, \delta_j)\big),
$$
and identity~\eqref{eq:amplification_identity} converts the left-hand side into $L(g_J)$. The product bound follows.
\end{proof}

\begin{remark}[Sampling requirement]
The conditional-i.i.d. assumption is the sampling condition needed for the
binomial calculation. It allows \(Q\), and the effective residual sizes \(Q_j\),
to be random. What matters is that, conditional on the realised size, the
retained positives are sampled from the appropriate conditional distribution.
In particular, positives obtained only by reviewing the included pool
\(\{g(X)=1\}\) are distributed as \(P(\cdot\mid Y=1,\,g(X)=1)\), not
\(P(\cdot\mid Y=1)\), and do not satisfy the hypothesis without an additional
reweighting or sampling argument. The same requirement applies to
Proposition~\ref{prop:direct_miss_rate} and to
Theorem~\ref{thm:stress_test}.

This also excludes audits in which the stopping rule, or the choice of which
positives to retain, depends on whether positives are caught by particular
components. Such dependence can change the conditional distribution of the
residual positives and invalidate the binomial law for \(H_j\mid Q_j\).
\end{remark}

\subsection{Residual shrinkage and the direct alternative}
\label{sec:residual_shrinkage_label}

The per-component bound $L_{Q_j}(H_j,\delta_j)$ on $\gamma_j$ is informative only when $Q_j$ is large. By construction, $Q_j$ counts the audited positives that the previous union $g_{j-1}$ has missed, so $Q_1\ge Q_2\ge\cdots\ge Q_J$: each added component reduces the residual pool of misses, and the audit count available for the next round shrinks accordingly. In the regime where the union already catches most relevant items, $Q_j$ is small for large $j$, the Clopper--Pearson lower bound $L_{Q_j}(H_j,\delta_j)$ is correspondingly weak, often close to zero, and the per-round contribution of $h_j$ is hard to certify even when the underlying $\gamma_j$ is large. The shrinkage is intrinsic to the decomposition: late-round detectors are exactly the ones whose contributions are hardest to demonstrate from positives alone.

The product bound on $L(g_J)$ produced by Theorem~\ref{thm:auditable_amplification} is not the tightest possible. The conditional analogue of Theorem~\ref{thm:excluded_binomial}, applied to the audited positives, gives a single direct certificate.

\begin{proposition}[Direct conditional miss-rate certificate]
\label{prop:direct_miss_rate}
Under the audit design of Theorem~\ref{thm:auditable_amplification}, let $C_J:=\#\{i\le Q:g_J(X_i^+)=0\}$ be the number of audited positives missed by the full union. Then, with probability at least $1-\delta$,
$$
L(g_J)\le U_Q(C_J,\delta),
$$
with the convention $U_0(0,\delta)=1$, so that an audit returning no positives ($Q=0$) gives the trivial bound $L(g_J)\le 1$.
\end{proposition}

\begin{proof}
Conditional on $Q$, the audited positives are i.i.d.\ from $P(\cdot\mid Y=1)$, so $C_J\sim\mathrm{Bin}(Q,L(g_J))$ with $L(g_J)=P(g_J(X)=0\mid Y=1)$. Lemma~\ref{lem:conditioning} then gives the displayed inequality.
\end{proof}

Proposition~\ref{prop:direct_miss_rate} is the natural overall guarantee: it
pays a single confidence allocation \(\delta\) rather than the split
\(\sum_j\delta_j\le\delta\), and it avoids the multiplicative accumulation of
slack across rounds. When \(J=1\) and \(\delta_1=\delta\), the two
certificates coincide, since \(U_n(k,\alpha)=1-L_n(n-k,\alpha)\) by the
symmetry \(\mathbb{P}(\mathrm{Bin}(n,p)\le k)=\mathbb{P}(\mathrm{Bin}(n,1-p)\ge n-k)\) of the
binomial tails. The value of Theorem~\ref{thm:auditable_amplification} is
diagnostic: it shows how coverage accumulates across the ordered components.
Late components may have small residual sample sizes \(Q_j\), making their
individual lower bounds weak.


\section{Stress-test cluster certificates}
\label{sec:stress}

In many applications the candidate generator should be tested for brittleness under pre-specified perturbations: paraphrases, OCR corruptions, spelling variants, or other declared perturbation mechanisms. We now give a finite-sample certificate relative to a fixed stress-test generator $\mathcal A$, which produces, for each item $x$, a set of variants $\mathcal A(x)\subseteq\mathcal X$. The certificate bounds the rate at which a relevant item has at least one variant in $\mathcal A(X)$ that $g$ excludes. We call this event a cluster escape. The scope of the certificate is fixed by $\mathcal A$ in advance: variants produced by mechanisms outside $\mathcal A$ are not covered, and the result is not a general semantic-robustness guarantee but a certificate of $g$'s sensitivity to the particular perturbation mechanism the analyst has chosen to test.

\subsection{Cluster generator and escape probability}

\begin{definition}[Stress-test generator]
\label{def:stress_generator}
A \emph{stress-test generator} is a measurable map $\mathcal A$ from $\mathcal X$ to finite non-empty subsets of $\mathcal X$ (or, more generally, to a distribution over finite non-empty subsets), fixed independently of the audit. The set $\mathcal A(x)$ is the \emph{stress-test cluster} of $x$.
\end{definition}

Examples include paraphrases generated by a fixed paraphrase model,
OCR-corruption variants produced by a fixed corruption process, spelling
variants generated by a declared rule, or perturbations produced by
another pre-specified mechanism. Whether $x$ itself belongs to
$\mathcal A(x)$ is left to the modeller.

A particularly relevant instance is the two-model semantic-adversarial
framework of \cite{AnthonySalehzadehNobari2026}. In that framework,
admissible paraphrase perturbations are constrained using a proxy
embedding, while their effect is evaluated through a distinct target
model. For a fixed perturbation budget $\varepsilon_{\mathrm{pert}}$
and a fixed discrete paraphrase-generation procedure, let
$\mathcal A_{\varepsilon_{\mathrm{pert}}}(x)$ denote the resulting
finite set, or distribution, of generated variants of $x$. This gives
a special case of the stress-test generator $\mathcal A$ defined above.
The certificates below then quantify either the probability that a
relevant item has at least one generated adversarial variant excluded
by $g$, or the expected fraction of its generated variants that are
excluded. These guarantees are relative to the declared generator
$\mathcal A_{\varepsilon_{\mathrm{pert}}}$. They do not by themselves
certify robustness over the full continuous perturbation set; such a
conclusion requires an additional covering or approximation condition.

\begin{definition}[Cluster escape probability]
\label{def:cluster_escape}
For a candidate generator $g$ and stress-test generator $\mathcal A$, define
\begin{align*}
\rho_{\max}(g)
&:=
\mathbb{P}_{X,\mathcal A}\left(
\exists z\in\mathcal A(X):g(z)=0
\,\middle|\,Y=1
\right),\\
\rho_{\mathrm{avg}}(g)
&:=
\mathbb E_{X,\mathcal A}\left[
\frac{1}{|\mathcal A(X)|}
\sum_{z\in\mathcal A(X)}
\mathbbm 1\{g(z)=0\}
\,\middle|\,Y=1
\right].
\end{align*}
The first is the probability that some variant in the cluster escapes the candidate set. The second is the expected fraction of variants in the cluster that escape. When $\mathcal A$ is randomised, the probability and expectation are taken over both $X\sim P(\cdot\mid Y=1)$ and an independent draw of the random cluster $\mathcal A(X)$.
\end{definition}

The two quantities answer different questions. The worst-cluster $\rho_{\max}$ controls the probability that some variant in a sampled cluster escapes. The average $\rho_{\mathrm{avg}}$ controls the expected fraction of variants that escape. The worst-cluster is the more stringent guarantee and scales adversely with cluster size. The average-variant is more forgiving and is the natural quantity for cluster-balanced evaluation.

\begin{theorem}[Stress-test cluster certificate]
\label{thm:stress_test}
Let $\mathcal G$ be a finite family of candidate generators and $\mathcal A$ a stress-test generator, both fixed independently of the audit. Let $X_1,\ldots,X_Q$ be audited positives which, conditional on $Q$, are i.i.d.\ from $P(\cdot\mid Y=1)$, and form clusters $\mathcal A(X_i)$ independently across $i$ (when $\mathcal A$ is randomised, its internal randomness is refreshed independently across $i$). For $g \in \mathcal G$, define
$$
W_i^{\max}(g) = \mathbbm 1\{\exists z \in \mathcal A(X_i) : g(z) = 0\}.
$$
Let $K^{\max}(g) = \sum_{i=1}^Q W_i^{\max}(g)$. In addition, draw $Z_1,\ldots,Z_Q$ so that, conditionally on the realised clusters
\[
\left(X_i,\mathcal A(X_i)\right)_{i=1}^Q,
\]
they are independent and each $Z_i$ is uniformly distributed on
$\mathcal A(X_i)$. The auxiliary randomisation used for these draws is independent of the
audit labels and of the internal randomness of $\mathcal A$. Let $K^{\mathrm{rand}}(g)=
\sum_{i=1}^Q \mathbbm 1\{g(Z_i)=0\}$. Then, with probability at least $1 - \delta$, simultaneously over $g \in \mathcal G$,
$$
\rho_{\max}(g) \le U_Q\left(K^{\max}(g), \frac{\delta}{2|\mathcal G|}\right), \qquad \rho_{\mathrm{avg}}(g) \le U_Q\left(K^{\mathrm{rand}}(g), \frac{\delta}{2|\mathcal G|}\right).
$$
Both bounds are read as the trivial value $1$ when $Q=0$.
\end{theorem}

\begin{proof}
The two bounds are handled separately. Throughout, $Q$ is treated as fixed. All probability statements below are conditional on $Q$, and the unconditional statement follows by averaging, as in Lemma~\ref{lem:conditioning}.

\textbf{Worst-cluster bound.}
For each $i=1,\ldots,Q$, the audited positive $X_i$ is i.i.d.\ from $P(\cdot\mid Y=1)$, and when $\mathcal A$ is randomised, its internal randomness at index $i$ is independent across $i$ by assumption. The indicator
$$
W_i^{\max}(g)=\mathbbm 1\{\exists z\in\mathcal A(X_i):g(z)=0\}
$$
is a measurable function of $X_i$ and of the independent cluster-generator randomness at index $i$. The variables $\{W_i^{\max}(g)\}_{i=1}^Q$ are therefore i.i.d., and each is Bernoulli with
$$
\mathbb E_{X,\mathcal A}[W_i^{\max}(g)]
=
\mathbb{P}_{X,\mathcal A}(\exists z\in\mathcal A(X_i):g(z)=0\mid Y_i=1)
=
\rho_{\max}(g),
$$
where the second equality is Definition~\ref{def:cluster_escape}. Hence
$$
K^{\max}(g)=\sum_{i=1}^Q W_i^{\max}(g)
\sim
\mathrm{Bin}(Q,\rho_{\max}(g)).
$$
The Clopper--Pearson upper bound for the binomial proportion satisfies, for every $\alpha\in(0,1)$,
$$
\mathbb{P}_{\mathrm{audit}}\left\{
\rho_{\max}(g)>U_Q(K^{\max}(g),\alpha)
\,\middle|\, Q
\right\}
\le\alpha.
$$
Applying this with $\alpha=\delta/(2|\mathcal G|)$ for each fixed $g$ and taking a union bound over $\mathcal G$ gives the simultaneous worst-cluster statement with failure probability at most $\delta/2$.

\textbf{Average-variant bound.} Conditional on $Q$, the triples $(X_i,\mathcal A(X_i),Z_i)$ are i.i.d.\ across $i$: the pairs $(X_i,\mathcal A(X_i))$ are i.i.d.\ by assumption, and the conditional law of $Z_i$ given the pair depends only on $\mathcal A(X_i)$. Writing $V_i(g):=\mathbbm 1\{g(Z_i)=0\}$ and averaging over the uniform choice of $Z_i$ within its cluster,
$$
\mathbb E_{X,\mathcal A}[V_i(g)]=\mathbb E_{X,\mathcal A}\left[\frac{1}{|\mathcal A(X_i)|}\sum_{z\in\mathcal A(X_i)}\mathbbm 1\{g(z)=0\}\,\middle|\,Y_i=1\right]=\rho_{\mathrm{avg}}(g),
$$
by Definition~\ref{def:cluster_escape}. Hence
$$
K^{\mathrm{rand}}(g)=\sum_{i=1}^Q V_i(g)\sim\mathrm{Bin}(Q,\rho_{\mathrm{avg}}(g)),
$$
and the Clopper--Pearson upper bound at level $\delta/(2|\mathcal G|)$ for each fixed $g$, with a union bound over $\mathcal G$, gives the simultaneous average-variant statement with failure probability at most $\delta/2$. The single draw $Z_i$ serves every $g\in\mathcal G$ at once, since it does not depend on $g$. Combining the two families of bounds gives the displayed joint statement with total failure probability at most $\delta$.
\end{proof}

\begin{remark}[Budget split between the two bounds]
\label{rem:stress_refinements}
The error budget need not be divided evenly between the two bounds. The allocation $\delta/2$ to each family may be replaced by any split $\delta_{\max}+\delta_{\mathrm{avg}}\le\delta$, fixed before the audit, with the larger share given to whichever guarantee is the operational target.
\end{remark}

\begin{remark}[Why draw one variant per cluster?]
\label{rem:exact_avg}
The single draw \(Z_i\) makes the average-variant certificate an exact binomial
certificate. Indeed, \(\mathbbm 1\{g(Z_i)=0\}\) is Bernoulli with mean
\(\rho_{\mathrm{avg}}(g)\). Thus no additive concentration term is needed.

This can matter in the rare-escape regime. For example, with \(Q=300\) and
error budget \(0.025\), zero observed escapes gives the exact bound
\(U_{300}(0,0.025)=0.0122\). By contrast, a Hoeffding bound for the empirical
mean of the within-cluster escape fractions has additive term
\[
\sqrt{\frac{\log(1/0.025)}{2Q}}=0.0784 .
\]
The empirical-mean route, based on
\(\widehat\rho_{\mathrm{avg}}(g)=Q^{-1}\sum_{i=1}^Q|\mathcal A(X_i)|^{-1}
\sum_{z\in\mathcal A(X_i)}\mathbbm 1\{g(z)=0\}\), is still available when all
variant values have been computed, but it is usually less sharp near zero.
\end{remark}

If the seed \(x\) is included in every cluster, then
\(\rho_{\max}(g)\ge L(g)\), since an ordinary miss is a cluster escape.
Otherwise the two quantities are not ordered.

The certificate is relative to the declared perturbation mechanism
\(\mathcal A\). It is not a general semantic-robustness guarantee: it certifies
escape only for the variants produced by \(\mathcal A\). Its role is to
distinguish ordinary coverage on positives drawn from \(P\) from brittleness
under the specified perturbations.

\section{Design-stage selection from a class}
\label{sec:learning}

Section~\ref{sec:design_iteration} allowed design choices before the certification labels are opened. We now give several formal ways to make such choices over finite, VC, and structured candidate-generator classes. Labelled design data may be used by a design-stage selector to choose a candidate generator $g$ before certification, using standard finite-class, VC, sparse-union, and Neyman--Pearson guarantees expressed in the one-sided coverage notation used here. The selected candidate generator should then be certified on an independent held-out audit using the certificates of Sections~\ref{sec:excluded}--\ref{sec:stress}. The design-stage selector is distinct from any downstream classifier trained or deployed after candidate generation.

For this section only, let $\mathcal H$ denote a class of measurable candidate generators $g:\mathcal X\to\{0,1\}$.\footnote{For infinite classes we assume throughout the standard measurability conditions under which the suprema of the relevant empirical processes are measurable.} A design sample from the whole population consists of i.i.d.\ pairs $(X_i,Y_i)_{i=1}^n$ with $X_i\sim P$ and $Y_i=\varphi^\star(X_i)$. Define
\begin{align*}
\widehat r_n(g) &= \frac1n\sum_{i=1}^n \mathbbm 1\{Y_i=1,\,g(X_i)=0\}, \\
\widehat B_n(g) &= \frac1n\sum_{i=1}^n \mathbbm 1\{g(X_i)=1\}.
\end{align*}
A candidate generator is \emph{positive-consistent} on the design sample if $\widehat r_n(g)=0$. It is \emph{label-consistent} if, in addition, it has no false positives on the sample.

\subsection{Why unrestricted coverage selection is vacuous}

The one-sided objective alone does not define a meaningful design-stage selection problem. Proposition~\ref{prop:trivial_coverage} gives the degeneracy: the constant candidate generator $g\equiv 1$ has zero missed mass, so an unrestricted coverage objective can be satisfied with no labelled data.

The two useful restrictions are the same ones used throughout the paper. First, the design-stage selector may be required to output $g\in\mathcal H$ for a fixed class $\mathcal H$. This is the standard \emph{proper-learning} restriction, meaning that the selector must output a candidate generator in the class it is searching over. Second, review burden may enter as an objective or constraint. This is the Neyman--Pearson formulation. In applications both restrictions are usually present: the analyst searches over a structured candidate-generator class and also imposes a burden target.

\subsection{Proper learning: finite classes and VC classes}

The following finite-class bound gives the basic sample-size rule for a design split. It is a one-sided version of the standard realisable PAC bound \cite{Valiant1984} for a design-stage selector.

\begin{theorem}[Finite-class design bound]
\label{thm:learning_finite_r}
Let $\mathcal H$ be finite, and suppose there exists $g^\star\in\mathcal H$ with $r(g^\star)=0$. From a whole-population design sample of size $n$, let $\widehat g\in\mathcal H$ be any positive-consistent candidate generator. Then, with probability at least $1-\delta$,
$$
r(\widehat g) \le \frac{\log(|\mathcal H|/\delta)}{n}.
$$
\end{theorem}

The conversion from absolute missed mass to conditional miss rate carries a prevalence penalty.

\begin{corollary}[Prevalence cost for recall]
\label{cor:learning_lpac_pi}
Under the assumptions of Theorem~\ref{thm:learning_finite_r}, assume $\pi=P(Y=1)>0$, where $\pi$ is the prevalence of relevant items. If
$$
n \ge \frac{\log(|\mathcal H|/\delta)}{\pi\varepsilon_L},
$$
then $L(\widehat g)\le \varepsilon_L$ with probability at least $1-\delta$.
\end{corollary}

For infinite classes, the cardinality term is replaced by VC complexity. For a binary class $\mathcal H$, its VC dimension is the largest integer $d$ for which there exist points $x_1,\ldots,x_d$ such that every binary labelling of those points is realised by some $h\in\mathcal H$. If no largest finite $d$ exists, the VC dimension is infinite. Finite VC dimension is the standard combinatorial condition that permits distribution-free uniform convergence over an infinite class.

\begin{theorem}[VC-class design bound]
\label{thm:learning_vc}
Let $\mathcal H$ have VC dimension $d$, suppose there exists $g^\star\in\mathcal H$ with $r(g^\star)=0$, and assume $n\ge d$. From a whole-population design sample of size $n$, let $\widehat g\in\mathcal H$ be any positive-consistent candidate generator. Then, with probability at least $1-\delta$,
$$
r(\widehat g) \le C\,\frac{d\log(en/d)+\log(1/\delta)}{n}
$$
for a universal constant $C$.
\end{theorem}

The constant $C$ can be made explicit by choosing a standard realisable VC inequality \cite{VapnikChervonenkis1971,BlumerEhrenfeuchtHausslerWarmuth1989}. Its exact value is not material for the audit certificates, which are computed on an independent held-out certification sample.

\subsection{Agnostic variants}

The realisable bounds above are clean baselines. In text and other ambiguous domains, the design class may contain no candidate generator with zero missed mass. Empirical missed-mass minimisation then gives the agnostic guarantees below, under the rule $\widehat g\in\arg\min_{g\in\mathcal H}\widehat r_n(g)$.

\begin{theorem}[Agnostic finite-class design bound]
\label{thm:learning_finite_agnostic}
Let $\mathcal H$ be finite and let $\widehat g\in\arg\min_{g\in\mathcal H}\widehat r_n(g)$. With probability at least $1-\delta$,
$$
r(\widehat g)\le \inf_{g\in\mathcal H}r(g)+2\sqrt{\frac{\log(2|\mathcal H|/\delta)}{2n}}.
$$
\end{theorem}

\begin{theorem}[Agnostic VC-class design bound]
\label{thm:learning_vc_agnostic}
Let $\mathcal H$ have VC dimension $d$, assume $n\ge d$, and let $\widehat g\in\arg\min_{g\in\mathcal H}\widehat r_n(g)$. With probability at least $1-\delta$, for a universal constant $C$,
$$
r(\widehat g)\le \inf_{g\in\mathcal H}r(g)+C\sqrt{\frac{d\log(en/d)+\log(1/\delta)}{n}}.
$$
\end{theorem}

These agnostic design bounds are often the operationally relevant ones. They still describe the design-stage selector and do not replace held-out certification of the realised candidate generator. We give the proofs and brief lower-bound context in Appendix~\ref{app:learning}.

Theorems~\ref{thm:learning_finite_r} and~\ref{thm:learning_vc} control the design procedure, not the realised generator by itself. They give high-probability guarantees over the random design sample. A data-dependent certificate for the selected generator still requires held-out certification labels, as in the earlier audit certificates.

\subsection{Burden-constrained selection}

The bicriteria formulation in Section~\ref{sec:bicriteria} suggests a Neyman--Pearson design rule: minimise review burden subject to a missed-mass constraint. Let
\begin{equation}
\label{eq:np_learning_problem}
\inf_{g\in\mathcal H} B(g) \quad \text{subject to} \quad r(g)\le \varepsilon
\end{equation}
be the population problem. The design rule is the empirical version of this optimisation problem: it replaces the population probabilities by their sample estimates and chooses a candidate generator minimising the resulting empirical objective:
\begin{equation}
\label{eq:np_learning_empirical}
\min_{g\in\mathcal H} \widehat B_n(g) \quad \text{subject to} \quad \widehat r_n(g)\le \varepsilon-\tau_n,
\end{equation}
where $\tau_n$ is a uniform-convergence slack.

\begin{theorem}[Neyman--Pearson design bound]
\label{thm:np_learning}
Let $\mathcal H$ have VC dimension $d$ and assume $n\ge d$. There is a universal constant $C>0$ such that the following holds. Let
$$
\tau_n=C\sqrt{\frac{d\log(en/d)+\log(2/\delta)}{n}}.
$$
Let $\widehat g$ solve~\eqref{eq:np_learning_empirical}, with $\widehat g$ defined arbitrarily in $\mathcal H$ if the empirical feasible set is empty. On the event used in the proof, the feasible set is non-empty. If there exists $g\in\mathcal H$ with $r(g)\le\varepsilon-2\tau_n$, then, with probability at least $1-\delta$,
\begin{enumerate}[label=(\roman*),leftmargin=2em]
    \item $r(\widehat g)\le\varepsilon$.
    \item $B(\widehat g)\le\inf\{B(g):g\in\mathcal H,\ r(g)\le\varepsilon-2\tau_n\}+2\tau_n$.
\end{enumerate}
\end{theorem}

Theorem~\ref{thm:np_learning} is the design-stage analogue of Neyman--Pearson classification \cite{ScottNowak2005,CannonEtAl2002,RigolletTong2011}. It applies when the analyst wants to tune a candidate generator before final certification, balancing observable review burden against a missed-mass constraint. The empirical rule in~\eqref{eq:np_learning_empirical} searches for a low-burden candidate generator, but enforces the stricter design constraint $\widehat r_n(g)\le \varepsilon-\tau_n$ to leave room for uniform-convergence error. The theorem then gives two guarantees: the selected candidate generator satisfies the target missed-mass constraint $r(\widehat g)\le\varepsilon$, and its burden is within $2\tau_n$ of the best population burden among candidate generators satisfying the margin condition $r(g)\le\varepsilon-2\tau_n$.

\subsection{Sparse unions of base detectors}

Candidate generators are often unions of weak detectors, such as keyword rules, regular-expression templates, embedding-neighbourhood filters, or score thresholds. If $\mathcal C$ is a base detector class, define
$$
\mathcal C^{\cup_k}=\left\{\max_{j\le k} c_j: c_j\in\mathcal C\right\},
$$
with repetitions allowed. Because the $c_j$ are $\{0,1\}$-valued, $\max_{j\le k} c_j$ is the indicator of $\bigcup_{j\le k}\{c_j=1\}$, so $\mathcal C^{\cup_k}$ is the class of $k$-fold unions of sets in $\mathcal C$. The standard VC composition bound \cite{BlumerEhrenfeuchtHausslerWarmuth1989} gives
$$
\mathrm{VC}(\mathcal C^{\cup_k})\le 2\,\mathrm{VC}(\mathcal C)\,k\log_{2}(3k).
$$
The logarithm here is to base two, and the $k\log k$ order is unavoidable: there are base classes with $\mathrm{VC}(\mathcal C^{\cup_k})=\Omega(d_{\mathcal C}\,k\log k)$~\cite{EisenstatAngluin2007}, so no composition bound of order $d_{\mathcal C}\,k$ holds in general.

\begin{theorem}[Sparse-union design bound]
\label{thm:sparse_union_learning}
Let $\mathcal C$ have VC dimension $d_{\mathcal C}\ge 1$, let $k\ge 1$ be an integer, and let $\mathcal H=\mathcal C^{\cup_k}$. Suppose there exists $g^\star\in\mathcal H$ with $r(g^\star)=0$, and assume $n\ge D$, where $D=2\,d_{\mathcal C}\,k\log_{2}(3k)$ bounds the VC dimension of $\mathcal H$ (the base of the logarithm, and the replacement of $\log(en/D)$ by $\log n$ in the displayed rate, affect only the constant $C$). From a whole-population design sample of size $n$, let $\widehat g\in\mathcal H$ be any positive-consistent candidate generator. Then, with probability at least $1-\delta$,
$$
r(\widehat g) \le C\,\frac{d_{\mathcal C}k\log(3k)\log n+\log(1/\delta)}{n}.
$$
\end{theorem}

\begin{remark}
If in addition some $g^\star\in\mathcal H$ satisfies $P(g^\star(X)\ne Y)=0$ and $\widehat g$ is label-consistent on the design sample, then the same VC argument applied to the symmetric-difference loss controls the total error $P(\widehat g(X)\ne Y)$ at the same rate, and hence bounds both $r(\widehat g)$ and the false-positive mass $b(\widehat g)$. We use this statement when precision, and not only review burden, must
be certified. The precision $1-\frac{b(g)}{B(g)}$, defined when $B(g)>0$, depends on both the false-positive mass $b(g)$ and the burden $B(g)$. In the finite-corpus model, $B(g)$ is known exactly from the generator values. In the population model, it must instead be estimated or lower-bounded using independent unlabelled
data. Thus, an upper confidence bound $\overline b(g)$ and a positive lower confidence bound $\underline B(g)$ give the certificate $\operatorname{Precision}(g)\ge
1-\frac{\overline b(g)}{\underline B(g)}$. The missed-mass bound alone does not constrain $b(g)$.
\end{remark}

\section{Worked example: audit planning and certification}
\label{sec:worked_example}

We now illustrate how to use the preceding certificates as an audit plan and reporting template. The numbers are deterministic calculations from the stated audit outcomes.

\needspace{14\baselineskip}
\paragraph{Practical audit design recipe.}
A typical certification workflow is:
\begin{center}
\fbox{\begin{minipage}{0.92\linewidth}
\begin{enumerate}[leftmargin=2em]
    \item Fix the candidate generator, candidate-generator family, or prefix sequence before opening the certification labels.
    \item Compute unlabelled quantities such as burden $B(g)$, excluded-pool mass $p_0(g)$, and pool memberships.
\item Choose the acceptable missed-mass threshold $\varepsilon$ and the overall certificate failure probability $\delta$. If several candidate generators are to be certified simultaneously, allocate failure probabilities $\delta_g>0$ satisfying
\[
\sum_{g\in\mathcal G}\delta_g\le\delta,
\]
for example, $\delta_g=\delta/|\mathcal G|$.
   \item Audit according to the sampling design matched to the desired certificate: excluded pool for absolute missed mass, both sides or a population sample for recall, and a shared reference sample for pre-specified prefixes.
    \item Convert the audit counts into simultaneous upper bounds $U(g)$.
    \item Report the least burdensome pre-specified candidate generator satisfying $U(g)\le \varepsilon$.
\end{enumerate}
\end{minipage}}
\end{center}

Suppose a corpus contains $N=100{,}000$ sentences. A candidate generator $g$ flags $30{,}000$ sentences for review and excludes $70{,}000$, so the observed burden is
$$
B_N(g)=0.30, \qquad p_{0,N}(g)=0.70.
$$
The audit target is a $95\%$ certificate. We reserve $\delta_0=0.025$ for the excluded-pool missed-mass certificate and $\delta_1=0.025$ for the included-pool lower bound used in the recall calculation.

\paragraph{Absolute missed mass.}
Draw $n_0=300$ items uniformly from the excluded pool and label them. If no relevant item is found ($K_0=0$), then
$$
U_{300}(0,0.025)=1-0.025^{1/300}\approx0.01222.
$$
The excluded-pool certificate gives
$$
r(g)\le 0.70\,U_{300}(0,0.025)\le 0.00856.
$$
Equivalently, in a corpus of $100{,}000$ items, the corresponding bound is approximately $855.47$ missed relevant sentences, or $856$ after upward rounding to a whole sentence. Thus the same audit certifies that the absolute missed relevant mass is below one percent of the corpus. The audit draws uniformly without replacement from a finite excluded pool, yet the calculation uses the binomial bound $U_{300}$. With zero observed misses this binomial bound is conservative under sampling without replacement. The exact hypergeometric inversion of Theorem~\ref{thm:excluded_hypergeometric} gives the slightly smaller finite-corpus bound $M_U=853$, that is $r(g)\le 0.00853$.

The certificate is sensitive to the observed audit count. Keeping $n_0=300$, $p_{0,N}(g)=0.70$, $N_0=70{,}000$, and $\delta_0=0.025$ fixed, the exact finite-corpus (hypergeometric) certified missed-mass bound becomes:
\begin{center}
\renewcommand{\arraystretch}{1.15}
\begin{tabular}{ccc}
\hline
Excluded-pool positives $K_0$ & Exact $M_U(K_0,0.025)$ & Certified upper bound on $r(g)$ \\
\hline
$0$ & $853$ & $0.00853$ \\
$1$ & $1{,}287$ & $0.01287$ \\
$2$ & $1{,}668$ & $0.01668$ \\
\hline
\end{tabular}
\end{center}
This illustrates why the zero-miss case is operationally valuable: even one or two missed relevant items in the audit can move the certificate above a one-percent missed-mass target.

\paragraph{Two-pool recall.}
To translate the missed-mass certificate into a recall statement, audit the included pool as well. Suppose $n_1=500$ included-pool items are labelled uniformly without replacement and $K_1=100$ are relevant. The exact finite-corpus lower inversion in Corollary~\ref{cor:two_pool_hypergeometric} gives
$$
M_L(100,0.025)=4{,}983.
$$
The included relevant mass is therefore bounded below by $4{,}983/100{,}000=0.04983$. Combining this lower bound with the exact excluded-pool upper bound $M_U=853$ gives the two-pool recall certificate
$$
\mathrm{Recall}(g)
=\frac{M_1(g)}{M_1(g)+M_0(g)}
\ge \frac{4{,}983}{4{,}983+853}\ge 0.8538.
$$
The reported statement is therefore: with simultaneous confidence at least $95\%$, the candidate generator has burden $30\%$, absolute missed relevant mass at most $0.853\%$ of the corpus, and recall at least $85.3\%$.

More generally, any finite collection of the certificates in this paper can be reported jointly by allocating error shares that sum to $\delta$ and taking a union bound, so a single audit can carry burden, missed-mass, recall, and stress-test statements about one fixed candidate generator at a stated joint confidence. Once these simultaneous bounds have been constructed, they may be compared with any number of targets, including targets chosen after the audit, without any further multiplicity adjustment. A fresh error allocation is required only for additional certificates concerning candidate generators or estimands not covered by the original jointly certified family.

\paragraph{Prefix choice.}
Now suppose the analyst pre-specifies three prefix unions $g_1,g_2,g_3$ before opening the certification audit. The audit goal is to choose the least burdensome prefix whose certified missed mass is at most $1\%$ of the corpus. Use separate excluded-pool audits with $n_m=300$ for each prefix and allocate $\delta_m=0.05/3$. If each excluded-pool audit observes no relevant item, then
$$
U_{300}(0,0.05/3)\approx0.01356.
$$
The burdens below are illustrative, chosen only to make the example concrete. The excluded masses are their complements $p_{0,N}(g_m)=1-B_N(g_m)$, and each certified bound is the excluded mass times this factor, $r(g_m)=p_{0,N}(g_m)\,\eta(g_m)\le p_{0,N}(g_m)\,U_{300}(0,0.05/3)$. The simultaneous certificates are:

\begin{center}
\renewcommand{\arraystretch}{1.2}
\begin{tabular}{ccccc}
\hline
Prefix & Burden $B_N(g_m)$ & Excluded mass $p_{0,N}(g_m)$ & $K_m$ & Certified $r(g_m)\le p_{0,N}(g_m)\,U_{300}$ \\
\hline
$g_1$ & $0.15$ & $0.85$ & $0$ & $0.0116$ \\
$g_2$ & $0.30$ & $0.70$ & $0$ & $0.00949$ \\
$g_3$ & $0.38$ & $0.62$ & $0$ & $0.00841$ \\
\hline
\end{tabular}
\end{center}

The valid post-audit choice is $g_2$: it is the least burdensome pre-specified prefix whose simultaneous certificate meets the $1\%$ missed-mass target. The third prefix gives a tighter certificate, but its additional eight percentage points of review burden are not needed for this target. This is the bicriteria formulation in operational form: the certificate controls missed mass, and burden selects among certified candidates.

\paragraph{Comparison with a whole-population audit.}
To match the bound $r(g)\le 0.00853$ from the excluded-pool calculation using a whole-population audit, sampling i.i.d.\ from $P$ rather than from $\{g=0\}$, with zero observed events and the same $\delta=0.025$ confidence allocation,
$$
n\ge \frac{\log(1/0.025)}{\log\big(1/(1-0.00853)\big)}\approx 431
$$
labels are required. The excluded-pool design achieves the same bound with $n_0=300$. The ratio $300/431\approx 0.70=p_{0,N}(g)$ matches the planning identity from Section~\ref{sec:samplesize}: at zero observed events, the whole-population audit needs about $1/p_0(g)$ times as many labels as the excluded-pool audit at the same target.

The excluded-pool gain in this example is therefore modest, because the excluded pool is large. The gain grows as the excluded pool shrinks: a hypothetical later candidate generator that flagged $80\%$ of the corpus, leaving $p_{0,N}(g)=0.20$, would require about five times more whole-population labels than excluded-pool labels for the same target. Beyond labelling cost, the excluded-pool design has an interpretive advantage independent of $p_0(g)$: every audited item is drawn from the only region in which missed relevant items can occur.

\section{Conclusion and future directions}
\label{sec:discussion}

This paper developed finite-sample certificates for missed relevant mass and
recall in high-recall candidate generation. The central point is that the
excluded pool is the only place where missed relevant items can occur. Auditing
that pool directly gives exact binomial and hypergeometric certificates for
missed mass, and these can be converted into recall guarantees when suitable
denominator information is available. The same framework also supports
simultaneous certification of pre-specified candidate-generator families, prefix
sequences, component decompositions, and stress tests.

The guarantees rely on a strict separation between design and certification.
Candidate generators may be built, tuned, and ordered using design data, but the
labels used for final certification must not have influenced the candidate
generator being certified. Once certification labels have been opened, they can
guide a later design round, but a fresh certification audit, or a pre-specified
sequential-validity device, is needed for the next certificate.

\paragraph{Stratified audits.}
The i.i.d.\ excluded-pool audit can be replaced by a stratified design that
splits the excluded pool into observable subgroups, or strata, and assigns more
audit labels to strata more likely to contain missed relevant items. A stratum
might group items by the candidate generator's underlying score, by document
category, or by time period. For a partition of the excluded pool into strata
$A_s$ with known masses $w_s=P(A_s\mid g(X)=0)$ and stratum-conditional counts
$K_0^{(s)}\sim\mathrm{Bin}\big(n_0^{(s)},\eta_s(g)\big)$, where
$\eta_s(g)=P(Y=1\mid X\in A_s,\,g(X)=0)$, a union bound gives, with probability
at least $1-\sum_s\delta_s$,
$$
\eta(g)=\sum_s w_s\,\eta_s(g)\le\sum_s w_s\,
U_{n_0^{(s)}}\big(K_0^{(s)},\delta_s\big),
$$
and hence the corresponding bound on $r(g)=p_0(g)\eta(g)$. This permits
oversampling of strata expected to contain misses while retaining a
distribution-free certificate. The finite-corpus version replaces the binomial
inversions by the hypergeometric inversions of
Theorem~\ref{thm:excluded_hypergeometric} within each stratum.

\paragraph{Label noise.}
The certification results above treat audit labels as ground truth. For the i.i.d.\ excluded-pool missed-mass certificate, suppose instead that the oracle returns a noisy label $\widetilde Y$ with per-item sensitivity
\[
s(X):=P(\widetilde Y=1\mid Y=1,X)\ge s_{\min}
\]
for a known $s_{\min}\in(0,1]$, and that the noisy labels are realised independently across audited items conditional on the audited items. Writing
\[
\widetilde\eta(g):=P(\widetilde Y=1\mid g(X)=0),
\]
we have
\[
\widetilde\eta(g)\ge s_{\min}\eta(g).
\]
Moreover,
\[
\widetilde K_0\sim\operatorname{Bin}\bigl(n_0,\widetilde\eta(g)\bigr),
\]
so the Clopper--Pearson bound gives, with probability at least $1-\delta$,
\[
r(g)\le p_0(g)\min\left\{1,\frac{U_{n_0}(\widetilde K_0,\delta)}{s_{\min}}\right\}.
\]
False-positive audit labels can only increase $\widetilde\eta(g)$ and therefore make this upper bound more
conservative.

This argument is specific to the i.i.d.\ excluded-pool missed-mass certificate. The two-pool recall certificate also requires a valid lower bound on
\[
\eta_1(g)=P(Y=1\mid g(X)=1).
\]
A sensitivity floor alone does not provide such a lower bound, because false-positive audit labels may inflate the observed positive rate in the included pool. Extending the two-pool recall certificate therefore
requires an additional assumption, such as a known specificity floor, or adjudicated labels. Likewise, Theorem~\ref{thm:auditable_amplification}, Proposition~\ref{prop:direct_miss_rate}, and
Theorem~\ref{thm:stress_test} sample conditional on true relevance $Y=1$. Replacing true relevance labels by noisy labels changes that conditional sampling distribution, so those results require additional
noise assumptions or adjudicated labels.

\paragraph{Further directions.}
Several directions remain open. One is to develop sharper sequential methods
for iterative audit-and-redesign workflows, so that certification can proceed
over multiple rounds without discarding information. A second is to combine
excluded-pool certification with richer models of label noise, adjudication, and
reviewer disagreement. A third is to extend the stress-test certificates to
structured perturbation spaces, where the declared perturbation mechanism
\(\mathcal A\) is replaced by geometric or semantic neighbourhoods.

The component-union perspective also raises computational questions. Since
coverage is naturally submodular in the selected components, one can ask for
efficient algorithms for choosing low-burden, high-coverage unions under
cardinality or cost constraints. This connects the certification problem to
submodular optimisation and knapsack-type selection~\cite{NemhauserWolseyFisher1978,
Sviridenko2004}, but we leave a systematic treatment to future work.

The overall message is that high-recall candidate generation can be certified
directly, without requiring a fully specified downstream classifier or a
full-corpus relevance audit. The price is discipline in the audit design: the candidate generator,
or the pre-specified family from which it is selected, and the audit
rule must be fixed before the certification labels are examined.

\appendix

\section{Learning-theoretic proofs and extended statements}
\label{app:learning}

This appendix gives the proofs for Section~\ref{sec:learning}. They are standard applications of finite-class union bounds, VC uniform convergence, and Neyman--Pearson plug-in selection. They are included to make the audit workflow self-contained.

\subsection{Proof of the finite-class design bound}

\begin{proof}[Proof of Theorem~\ref{thm:learning_finite_r}]
Set $\varepsilon=\log(|\mathcal H|/\delta)/n$. Fix $g\in\mathcal H$ with $r(g)>\varepsilon$. On a whole-population design sample, the probability that $g$ is positive-consistent is
$$
\mathbb{P}_{\mathrm{design}}\{\widehat r_n(g)=0\}=(1-r(g))^n\le e^{-n\varepsilon}.
$$
Therefore
$$
\mathbb{P}_{\mathrm{design}}\left\{
\exists g\in\mathcal H:
\widehat r_n(g)=0,\ r(g)>\varepsilon
\right\}
\le |\mathcal H|e^{-n\varepsilon}=\delta.
$$
On the complement of this event, every positive-consistent candidate generator in $\mathcal H$ has missed mass at most $\varepsilon$. The assumption that some $g^\star\in\mathcal H$ has $r(g^\star)=0$ ensures that at least one positive-consistent member exists.
\end{proof}

\begin{proof}[Proof of Corollary~\ref{cor:learning_lpac_pi}]
The identity $r(g)=\pi L(g)$ holds for every candidate generator when $\pi=P(Y=1)>0$. Theorem~\ref{thm:learning_finite_r} gives $r(\widehat g)\le \log(|\mathcal H|/\delta)/n$ with probability at least $1-\delta$. The displayed lower bound on $n$ therefore implies $L(\widehat g)\le\varepsilon_L$.
\end{proof}

The exact zero-observed-miss calculation gives the alternative sufficient condition
$$
n\ge \frac{\log(|\mathcal H|/\delta)}{\log(1/(1-\pi\varepsilon_L))},
$$
which is slightly sharper for non-small $\pi\varepsilon_L$.

\subsection{Proof of the VC-class design bound}

\begin{proof}[Proof of Theorem~\ref{thm:learning_vc}]
Consider the loss class
$$
\mathcal L_r=\{(x,y)\mapsto \mathbbm 1\{y=1,\ g(x)=0\}:g\in\mathcal H\}.
$$
Its VC dimension is at most the VC dimension $d$ of $\mathcal H$, since it is the class of complements of candidate-generator sets restricted to the slice $y=1$. The standard realisable VC bound \cite{VapnikChervonenkis1971,BlumerEhrenfeuchtHausslerWarmuth1989} says that, with probability at least $1-\delta$, every member of $\mathcal L_r$ with zero empirical loss has population loss at most
$$
C\,\frac{d\log(en/d)+\log(1/\delta)}{n}.
$$
A positive-consistent $\widehat g$ has zero empirical loss in $\mathcal L_r$, so the result follows.
\end{proof}

\subsection{Agnostic variants and lower bounds}

\begin{proof}[Proof of Theorem~\ref{thm:learning_finite_agnostic}]
Write $s=\sqrt{\log(2|\mathcal H|/\delta)/(2n)}$. For $\widehat g\in\arg\min_{g\in\mathcal H}\widehat r_n(g)$, Hoeffding's inequality and a union bound over $\mathcal H$ give, with probability at least $1-\delta$,
$$
\sup_{g\in\mathcal H}|r(g)-\widehat r_n(g)|\le s.
$$
On that event $r(\widehat g)\le \widehat r_n(\widehat g)+s$. Since $\widehat g$ minimises $\widehat r_n$ and $\widehat r_n(g)\le r(g)+s$ for every $g$, we have $\widehat r_n(\widehat g)\le \inf_{g\in\mathcal H}r(g)+s$, and the stated bound $r(\widehat g)\le\inf_{g\in\mathcal H}r(g)+2s$ follows.
\end{proof}

This controls missed mass alone. Burden can be handled by class design, by a separate burden constraint, or by the Neyman--Pearson rule of Theorem~\ref{thm:np_learning}.

\begin{proof}[Proof of Theorem~\ref{thm:learning_vc_agnostic}]
The same argument with the uniform-convergence inequality for a VC class of dimension $d$ in place of the finite-class union bound gives, with probability at least $1-\delta$,
$$
\sup_{g\in\mathcal H}|r(g)-\widehat r_n(g)|\le C\sqrt{\frac{d\log(en/d)+\log(1/\delta)}{n}},
$$
and hence $r(\widehat g)\le \inf_{g\in\mathcal H}r(g)+C\sqrt{(d\log(en/d)+\log(1/\delta))/n}$ after absorbing the factor two into $C$.
\end{proof}

The confidence term cannot be improved in distribution-free zero-miss certification. Any procedure that observes zero misses in $n$ Bernoulli trials and certifies miss probability at most $\varepsilon$ with confidence $1-\delta$ must have
$$
n\ge \frac{\log(1/\delta)}{\log(1/(1-\varepsilon))}.
$$
Otherwise, under a Bernoulli law with miss probability $p>\varepsilon$, the zero-miss event has probability $(1-p)^n>\delta$ for $p$ close enough to $\varepsilon$, making the certificate invalid. The class-complexity term $(\log|\mathcal H|)/\varepsilon$ is also unavoidable in the worst case: there are finite classes of cardinality $|\mathcal H|$ with VC dimension of order $\log |\mathcal H|$, and the standard PAC lower bound of Ehrenfeucht, Haussler, Kearns, and Valiant applies to such classes~\cite{EHKV1989}.

\subsection{Proof of the Neyman--Pearson design bound}

\begin{proof}[Proof of Theorem~\ref{thm:np_learning}]
The same VC dimension controls both uniform deviations used below. For \(B(g)\),
the relevant class is
\[
\{x:g(x)=1\},\qquad g\in\mathcal H,
\]
which has VC dimension \(d\). For \(r(g)\), the relevant class is
\[
\{(x,y):y=1,\ g(x)=0\},\qquad g\in\mathcal H .
\]
This class has VC dimension at most \(d\), since points with \(y=0\) can never
be labelled \(1\), while on points with \(y=1\) it is the complement class of
\(\mathcal H\). Thus a single VC uniform-convergence bound, with a union bound
over the two classes, controls both \(r\) and \(B\).

By the definition of $\tau_n$, with probability at least $1-\delta$ the event
$$
\mathcal E=\left\{\sup_{g\in\mathcal H}|r(g)-\widehat r_n(g)|\le\tau_n,
\quad \sup_{g\in\mathcal H}|B(g)-\widehat B_n(g)|\le\tau_n\right\}
$$
holds. Work on $\mathcal E$. The hypothesis supplies a $g$ with $r(g)\le\varepsilon-2\tau_n$, hence $\widehat r_n(g)\le r(g)+\tau_n\le\varepsilon-\tau_n$, so the empirical feasible set $\{g:\widehat r_n(g)\le\varepsilon-\tau_n\}$ is non-empty and $\widehat g$ is well-defined. Since $\widehat g$ is empirically feasible, meaning $\widehat r_n(\widehat g)\le\varepsilon-\tau_n$,
$$
r(\widehat g)\le \widehat r_n(\widehat g)+\tau_n\le \varepsilon.
$$
Now let $g^\circ$ be any candidate generator satisfying $r(g^\circ)\le\varepsilon-2\tau_n$. Then
$$
\widehat r_n(g^\circ)\le r(g^\circ)+\tau_n\le\varepsilon-\tau_n,
$$
so $g^\circ$ is empirically feasible. Since $\widehat g$ minimises empirical burden over the empirical feasible set,
$$
\widehat B_n(\widehat g)\le \widehat B_n(g^\circ)\le B(g^\circ)+\tau_n.
$$
Converting back to population burden gives
$$
B(\widehat g)\le \widehat B_n(\widehat g)+\tau_n\le B(g^\circ)+2\tau_n.
$$
Taking the infimum over all $g^\circ$ with $r(g^\circ)\le\varepsilon-2\tau_n$ proves the claim.
\end{proof}

\subsection{Proof of the sparse-union design bound}

\begin{proof}[Proof of Theorem~\ref{thm:sparse_union_learning}]
The class $\mathcal C^{\cup_k}$ has VC dimension at most $2d_{\mathcal C}k\log_{2}(3k)$ by the standard union-composition bound \cite{BlumerEhrenfeuchtHausslerWarmuth1989}. The missed-mass loss class
$$
\{(x,y)\mapsto \mathbbm 1\{y=1,\,g(x)=0\}:g\in\mathcal C^{\cup_k}\}
$$
has VC dimension at most that of $\mathcal C^{\cup_k}$, since restricting to the slice $\{y=1\}$ cannot increase the number of realisable labellings. Positive-consistency makes the empirical missed-mass loss of $\widehat g$ zero, so the realisable VC bound of Theorem~\ref{thm:learning_vc}, applied to $\mathcal H=\mathcal C^{\cup_k}$ with $d\le 2d_{\mathcal C}k\log_{2}(3k)$, gives
$$
r(\widehat g)\le C\,\frac{d_{\mathcal C}k\log(3k)\log n+\log(1/\delta)}{n}
$$
with probability at least $1-\delta$. For the remark, label-consistency makes the empirical symmetric-difference loss zero. The same realisable VC bound, applied to the loss class $\{(x,y)\mapsto\mathbbm 1\{g(x)\ne y\}\}$ of the same VC order, controls $P(\widehat g(X)\ne Y)$, and $r(\widehat g)$ and $b(\widehat g)$ are its two disjoint components.
\end{proof}

\end{document}